\documentclass{article}
\usepackage{times}
\usepackage[margin=1in]{geometry}

\usepackage[authoryear,round]{natbib}

\usepackage{amsmath,amsfonts,bm}

\def\eqref#1{equation~\ref{#1}}
\def\Eqref#1{Equation~\ref{#1}}

\def\1{\bm{1}}

\DeclareMathAlphabet{\mathsfit}{\encodingdefault}{\sfdefault}{m}{sl}
\SetMathAlphabet{\mathsfit}{bold}{\encodingdefault}{\sfdefault}{bx}{n}

\bibpunct{(}{)}{;}{a}{,}{,}

\usepackage{booktabs}
\usepackage{graphicx}

\usepackage{enumitem}

\usepackage{hyperref}
\usepackage{url}
\usepackage{xcolor}
\usepackage{adjustbox}
\usepackage{algorithm}
\usepackage{subcaption}
\usepackage{algpseudocode}
\usepackage{multirow}

\definecolor{GeoRed}{RGB}{180,45,45}

\definecolor{boxyellow}{HTML}{F7D476}
\definecolor{boxred}{HTML}{DA755C}
\definecolor{boxgreen}{HTML}{489D8B}

\algrenewcommand\algorithmicrequire{\textbf{Input:}}
\algrenewcommand\algorithmicensure{\textbf{Output:}}
\usepackage{amsmath}

\usepackage{amssymb}

\usepackage{amsthm}
\usepackage{float}
\usepackage{enumitem}

\newcommand{\gsnorm}[1]{\left\|#1\right\|_2}
\newcommand{\gsip}[2]{\left\langle #1,#2\right\rangle}
\newcommand{\gsE}{\mathbb E}
\newcommand{\gsR}{\mathbb R}
\newcommand{\gsclip}{\operatorname{clip}_{[0,1]}}
\DeclareMathOperator{\gsspan}{span}
\newtheorem{gstheorem}{Theorem}
\newtheorem{gsproposition}{Proposition}
\newtheorem{gslemma}{Lemma}
\newtheorem{gscorollary}{Corollary}
\theoremstyle{remark}

\floatstyle{ruled}
\newfloat{gsalgorithm}{tbp}{gsa}
\floatname{gsalgorithm}{Algorithm}

\title{GeoShrink: Accelerating Diffusion Transformers\\ with Two Lines of Code}

\author{
Haosen Li$^{1,*}$ \quad
Wenshuo Chen$^{1,*}$ \quad
Shaofeng Liang$^{1}$ \quad
Lei Wang$^{2,3}$ \quad
Bowen Tian$^{1}$ \quad
Yutao Yue$^{1,4,\dagger}$ \\[0.6em]
\small $^{1}$The Hong Kong University of Science and Technology (Guangzhou) \\
\small $^{2}$Griffith University 
\small $^{3}$Data61, CSIRO 
\small $^{4}$Artificial Intelligence Lab, Institute of Deep Perception Technology, JITRI \\[0.4em]
\small
\texttt{jasonhaosenl@hkust-gz.edu.cn} \quad
\texttt{wenshuochen@hkust-gz.edu.cn} \\
\small
\texttt{shawnsfliang@hkust-gz.edu.cn} \quad
\texttt{l.wang4@griffith.edu.au} \\
\small
\texttt{bowentian@hkust-gz.edu.cn} \quad
\texttt{yutaoyue@hkust-gz.edu.cn} \\[0.4em]
\small $^{*}$Equal contribution \qquad
\small $^{\dagger}$Corresponding author
}

\date{}

\begin{document}
\raggedbottom

\maketitle

\begin{figure}[h]
\centering
\includegraphics[width=\linewidth]{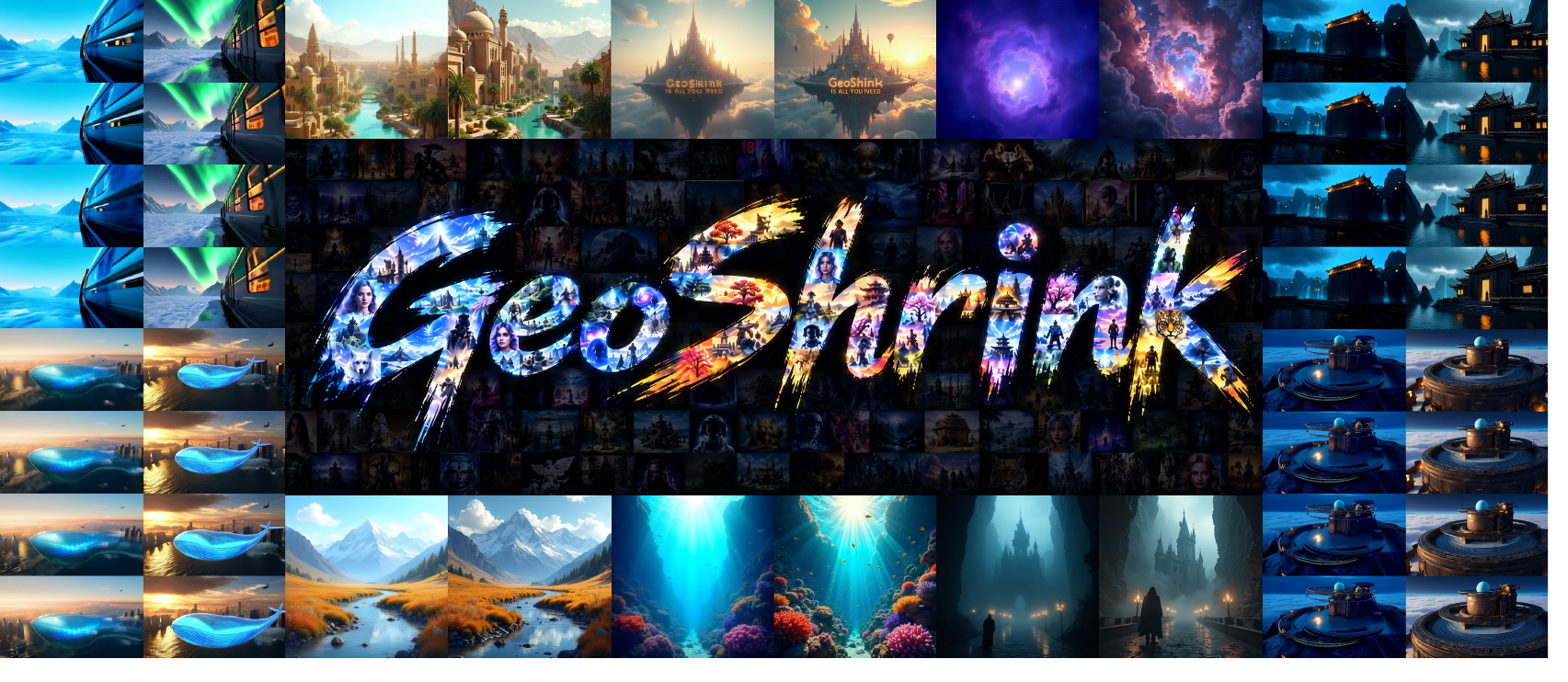}
\caption{\textbf{Qualitative comparison of GeoShrink.}
Direct NFE reduction (left) vs.\ GeoShrink (right) under the same model-evaluation budget. GeoShrink better preserves visual structure and temporal evolution under aggressive acceleration.}
\label{fig:placeholder}
\end{figure}

\begin{abstract}
Diffusion transformers incur substantial inference cost through repeated model
evaluations along a sampling trajectory. We introduce GeoShrink, a training-free
acceleration method that retains the original solver grid while evaluating the
model only at a prescribed set of anchors. At skipped stages, GeoShrink predicts
the solver-facing output by adding a geometrically retained fraction of the
latest observed innovation to the most recent exact output. 
We derive this rule from chordal tangent transport and round-trip line
projection, and establish a geometric anchor-spacing principle that minimizes
the largest adjacent gap expansion under fixed coverage and first span.
The analysis characterizes the geometric closure and propagation of prediction
errors without assuming access to future model outputs. Experiments cover
image, video, motion, and audio generation, together with adapted 3D backends.
At approximately $5\times$ acceleration, GeoShrink improves FLUX PSNR by
$3.10$\,dB over the strongest listed baseline. On HunyuanVideo, it achieves a
reported $4.99\times$ speedup and improves ChronoMagic-Bench-150 PSNR by
$5.44$\,dB over the strongest listed fidelity baseline. Comparisons at fixed
evaluation budgets further show substantial gains on Motion, Audio, Music and 3d generation.
\end{abstract}

\section{Introduction}
\label{sec:introduction}

Diffusion transformers~\citep{peebles2023dit,bao2023uvit} have become a common backbone for generative modeling
across images~\citep{esser2024sd3,blackforestlabs2024flux}, video~\citep{kong2024hunyuanvideo}, human motion~\citep{wen2025hymotion}, audio~\citep{hung2024tangoflux,gong2025acestep}, and 3D content~\citep{xiang2025trellis,zhao2025hunyuan3d2}.
Despite their versatility, inference remains expensive because a large neural
network should be repeatedly evaluated along the sampling trajectory~\citep{sohldickstein2015diffusion,ho2020ddpm,song2019score,song2021sde,rombach2022ldm}. Reducing
these evaluations without modifying or retraining the pretrained generator is
therefore an important practical problem across generative domains.

Existing training-free acceleration methods exploit temporal redundancy through caching or forecasting intermediate features across denoising steps~\citep{ma2024deepcache,selvaraju2024fora,zhao2024pab,liu2025teacache,zou2025toca,lv2025fastercache,liu2025taylorseer,liu2025speca,zhao2026resilphase}.
Feature caching reuses previously computed representations, while feature forecasting extrapolates them to approximate features at subsequent stages.
While effective, these approaches rely on model-specific internal representations, making transfer to new backbones dependent on choices such as layers, tokens, cache locations, reuse intervals, or verification mechanisms~\citep{zou2025toca,liu2025speca}.
We explore a simpler alternative: accelerating diffusion models entirely at the \emph{solver interface}.
The solver only consumes the model output at each sampling stage, rather than its internal representations.
We therefore retain the original solver grid, query the pretrained model only at sparse anchor stages, and predict the missing solver-facing outputs from previous exact observations.
The key question is then: \emph{how much of the latest observed change should be reused?}
As shown in Figure~\ref{fig:motivation}, small coefficients underuse the observed trajectory, whereas full reuse or further extrapolation quickly increases deviation from the full-NFE reference.
This motivates shrinking the latest change rather than predicting a new high-dimensional representation.

\begin{figure}[t]
    \centering
    \includegraphics[width=\linewidth]{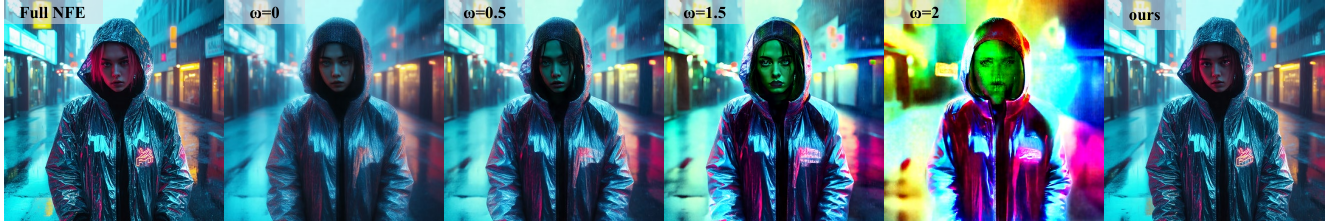}
    \caption{
    \textbf{Motivation for adaptive innovation retention.}
    From left to right, we show the full-NFE reference,
    fixed-coefficient predictions with
    \(\omega\in\{0,0.5,1.5,2\}\), and GeoShrink.
    Small \(\omega\) underuses the latest innovation, while aggressive reuse
    (\(\omega>1\)) progressively introduces color, contrast, and appearance drift.
    GeoShrink instead adapts \(\omega_t\) from observable trajectory geometry
    and more closely preserves the full-NFE reference.
    }
    \label{fig:motivation}
\end{figure}

Motivated by this observation, we introduce \textbf{GeoShrink}, a training-free acceleration method operating entirely at the solver interface. GeoShrink stores only the three most recent exact outputs and predicts along the latest observed change. Its key idea is to adaptively shrink this change using trajectory geometry: stronger turning or longer forecast horizons lead to more conservative reuse, while locally consistent changes over short horizons retain more of the latest direction. This reduces prediction from forecasting a new high-dimensional representation to estimating the reliability of an already observed direction.
The same principle also determines where exact queries are placed. Since future turning is unknown, GeoShrink controls prediction exposure through anchor spacing. Under a fixed query budget, minimizing the worst-case exposure yields a geometric anchor schedule. Thus, GeoShrink unifies \emph{how much} observed information to reuse with \emph{how far} to propagate it, without training, learned predictors, intermediate-feature access, or architecture-specific design.

On standard image and video generation, GeoShrink substantially outperforms
strong training-free acceleration baselines: at approximately \(5\times\)
acceleration, it improves PSNR by \(3.10\) dB on FLUX.1-dev and \(5.44\) dB
on HunyuanVideo over the strongest listed fidelity baselines.
Transferring the same solver-interface principle to human motion yields even
larger gains, reducing key geometric errors by tens to hundreds of times over
direct step reduction at matched NFE budgets.
We further extend GeoShrink to 3D, audio, and music generation, demonstrating
broad cross-modal applicability across diverse architectures and generative domains.
Our contributions are threefold:
\renewcommand{\labelenumi}{\roman{enumi}.}
\begin{enumerate}[leftmargin=0.6cm]
\item
We formulate training-free acceleration as sparse prediction of solver-facing
outputs while preserving the original solver grid, avoiding model-internal
feature access.

\item
We derive a parameter-free geometric retention rule and a fixed-budget
 anchor schedule whose unique continuous solution is a geometric mesh.

\item
We analyze prediction and error propagation, and evaluate GeoShrink across
image, video, motion, audio, music, and 3D generation.

\end{enumerate}

\begin{figure}[t]
\centering
\includegraphics[width=\linewidth]{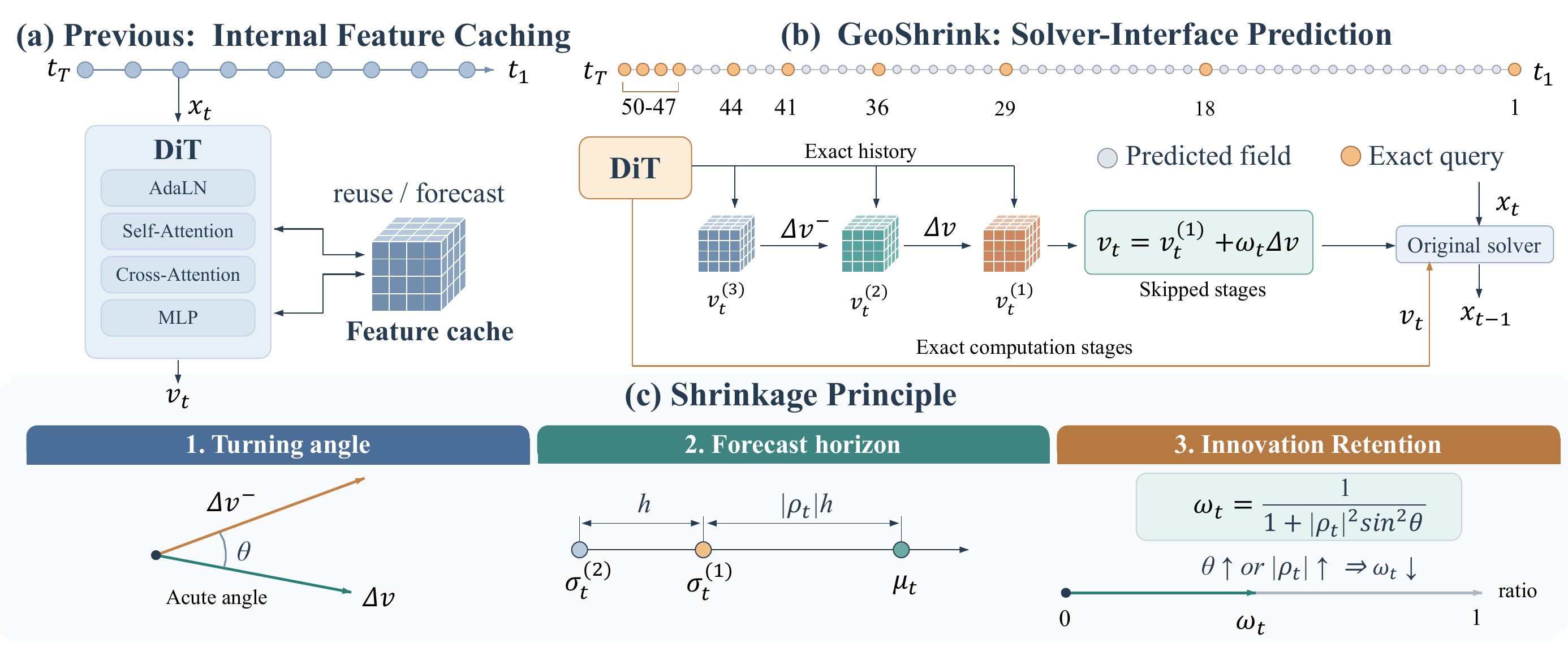}
\caption{\textbf{Overview of GeoShrink.}
(a) Previous methods rely on internal feature caching, reusing or forecasting intermediate DiT representations across denoising steps.
(b) GeoShrink operates at the solver interface, allocating 10 exact model queries over a 50-step grid through geometric anchoring. At skipped stages, it predicts the solver-facing field from the three latest exact outputs while preserving the original solver and its grid.
(c) The retention coefficient $\omega_t=(1+\rho_t^2\sin^2\theta)^{-1}$ combines the acute angle $\theta$ between successive innovations with the normalized midpoint forecast horizon $\rho_t$. Stronger turning or a longer horizon reduces the retained innovation in $v_t=v_t^{(1)}+\omega_t\Delta v$.}
\label{fig:pipeline}
\end{figure}

\section{Methodology}
\label{sec:method}
\label{sec:gs-method}

\subsection{Preliminaries}
\label{sec:preliminaries}

\textbf{Diffusion transformers}~\citep{peebles2023dit} process noisy latent tokens conditioned on time and auxiliary information.
Rectified flow~\citep{liu2023rectifiedflow} learns a velocity field
\(V_\theta(x_\sigma,\sigma,c)\) with target \(\epsilon-x_{\mathrm{data}}\)
along \(x_\sigma=(1-\sigma)x_{\mathrm{data}}+\sigma\epsilon\),
where \(x_{\mathrm{data}}\) is a clean latent,
\(\epsilon\sim\mathcal N(0,I)\), and \(c\) is the condition.
Sampling integrates \(\mathrm{d}x/\mathrm{d}\sigma=V_\theta(x,\sigma,c)\)
from \(\sigma=1\) to \(\sigma=0\).
On the grid \(\{\sigma_t\}_{t=0}^{T}\), a solver uses
\(v_t=V_\theta(x_t,\sigma_t,c)\) to update \(x_t\) to \(x_{t-1}\).

\subsection{Geometric Retention from Exact Anchors}
\label{sec:gs-retention}

At a non-anchor stage \(t\notin\mathcal A\), let
\(t^{(3)}>t^{(2)}>t^{(1)}>t\) index the three latest exact queries, with outputs
\(v_t^{(i)}=V_\theta(x_{t^{(i)}},\sigma_t^{(i)},c)\),
\(i\in\{1,2,3\}\), where \(\sigma_t^{(i)}=\sigma_{t^{(i)}}\). Define the innovations
\(\Delta v=v_t^{(1)}-v_t^{(2)}\) and
\(\Delta v^{-}=v_t^{(2)}-v_t^{(3)}\).
GeoShrink uses the prediction family
\begin{equation}
v_t(\omega)=v_t^{(1)}+\omega \Delta v,
\qquad 0\le \omega\le1,
\label{eq:gs-family}
\end{equation}
Let \(\theta\in[0,\pi/2]\) be the angle between the innovation lines:
\begin{equation}
\sin^2\theta
=
\sin^2\angle(\Delta v,\Delta v^{-})
=
\frac{\|P_{\Delta v^{-}}^{\perp}\Delta v\|_2^2}{\|\Delta v\|_2^2}
\in[0,1].
\label{eq:gs-chi}
\end{equation}
With \(\mu_t=(\sigma_t+\sigma_{t-1})/2\) and \(h=|\sigma_t^{(1)}-\sigma_t^{(2)}|\), define the normalized horizon
\begin{equation}
\rho_t
=
\frac{\mu_t-\sigma_t^{(1)}}
     {\sigma_t^{(1)}-\sigma_t^{(2)}}
=
\frac{\sigma_t+\sigma_{t-1}-2\sigma_t^{(1)}}
     {2(\sigma_t^{(1)}-\sigma_t^{(2)})}.
\label{eq:gs-rho}
\end{equation}
The geometric exposure is \(R_t=\rho_t^2\sin^2\theta\).

\begin{gsproposition}
\label{prop:gs-retention}
Assume \(\Delta v\neq0\) and \(\Delta v^{-}\neq0\). Let
\(u_t=\Delta v/\|\Delta v\|_2\) and \(u_t^{-}=\Delta v^{-}/\|\Delta v^{-}\|_2\), choosing the sign of
\(u_t^{-}\) such that \(\langle u_t,u_t^{-}\rangle\ge0\). Define
\(\xi_t=P_{u_t}^{\perp}u_t^{-}\),
\(\eta_t=-\rho_t\xi_t\), and
\(\widetilde u_t=(u_t+\eta_t)/\|u_t+\eta_t\|_2\).
The line overlap is
\begin{equation}
\omega_t
=
\operatorname{tr}(P_{u_t}P_{\widetilde u_t})
=
\frac{1}{1+\rho_t^2\sin^2\theta}
\label{eq:gs-retention}
\end{equation}
\end{gsproposition}

\begin{proof}
By construction,
\(\langle u_t,\xi_t\rangle=0\) and \(\|\xi_t\|_2^2=\sin^2\theta\). Hence
\[
\|\eta_t\|_2^2=\rho_t^2\sin^2\theta=R_t,
\qquad
\|u_t+\eta_t\|_2^2=1+R_t,
\qquad
\langle u_t,\widetilde u_t\rangle^2=\frac{1}{1+R_t}.
\]
Since \(P_{u_t}\) and \(P_{\widetilde u_t}\) are rank-one projectors,
\(\omega_t=\operatorname{tr}(P_{u_t}P_{\widetilde u_t})
=\langle u_t,\widetilde u_t\rangle^2\), proving the claim.
\end{proof}

Combining Proposition~\ref{prop:gs-retention} and \eqref{eq:gs-family} yields
\begin{equation}
{
v_t
=
v_t^{(1)}
+
\frac{
v_t^{(1)}-v_t^{(2)}
}{
1+
\left(
\frac{
\sigma_t+\sigma_{t-1}-2\sigma_t^{(1)}
}{
2(\sigma_t^{(1)}-\sigma_t^{(2)})
}
\right)^2
\sin^2\!\angle\!\left(
v_t^{(1)}-v_t^{(2)},
v_t^{(2)}-v_t^{(3)}
\right)
}
}
\label{eq:geoshrink_predictor}
\end{equation}

\subsection{Budget-Optimal Geometric Anchoring}
\label{sec:geo_schedule}
\label{sec:gs-schedule}

With future turning unknown, anchor spacing controls exposure.
Let \(h_j,h_{j+1}>0\) denote consecutive exact-query gaps in a forward progress coordinate.

\begin{gslemma}
\label{lem:gs-gap-ratio}
For skipped updates in the cell following \(h_j\),
\begin{equation}
0\le|\rho_t^{(s)}|\le\frac{h_{j+1}}{h_j}.
\label{eq:gs-horizon-bound}
\end{equation}
If \(h_{j+1}/h_j\le\Gamma\), then
\begin{equation}
R_t^{(s)}\le\Gamma^2,
\qquad
\omega_t^{(s)}\ge\frac{1}{1+\Gamma^2}.
\label{eq:gs-retention-certificate}
\end{equation}
\end{gslemma}

\begin{proof}
The distance from the latest anchor to any point in the next cell is at most
\(h_{j+1}\), while the normalization span is \(h_j\), giving
\(|\rho_t^{(s)}|\le h_{j+1}/h_j\). Since \(0\le\sin^2\theta\le1\),
\(R_t^{(s)}=(\rho_t^{(s)})^2\sin^2\theta\le\Gamma^2\), and
\eqref{eq:gs-retention} gives the retention bound.
\end{proof}

With \(B\) queries over \(T\) stages and bootstrap queries
(\(h_1=h_2=1\)), Lemma~\ref{lem:gs-gap-ratio} motivates
\begin{equation}
\min_{\{h_j\}_{j=1}^{B-1}}
\max_{2\le j\le B-2}\frac{h_{j+1}}{h_j}
\quad
\text{s.t.}\quad
h_1=h_2=1,
\qquad
\sum_{j=1}^{B-1}h_j=T-1.
\label{eq:gs-gap-minimax}
\end{equation}

\begin{gstheorem}
\label{thm:gs-geometric-main}
\Eqref{eq:gs-gap-minimax} has the unique continuous solution
\begin{equation}
h_j=r_\star^{\max(j-2,0)},
\qquad
j=1,\ldots,B-1,
\qquad
\sum_{j=1}^{B-1}r_\star^{\max(j-2,0)}=T-1,
\quad r_\star\ge1.
\label{eq:gs-geometric}
\end{equation}
\end{gstheorem}

\begin{proof}
Let \(\Gamma=\max_{2\le j\le B-2} h_{j+1}/h_j\). Then \(h_j\le\Gamma^{\max(j-2,0)}\), so any feasible schedule satisfies

\[
T-1
=
\sum_{j=1}^{B-1}h_j
\le
\sum_{j=1}^{B-1}\Gamma^{\max(j-2,0)}.
\]
The right-hand side is strictly increasing for \(\Gamma\ge1\), so
\(\Gamma\ge r_\star\). The geometric sequence attains equality and is therefore
optimal. If an optimizer satisfied
\(h_j<r_\star^{\max(j-2,0)}\) for any \(j\), its total coverage would be
strictly less than \(T-1\); hence equality holds for every \(j\), proving
uniqueness.
\end{proof}
The continuous anchor stages are
\begin{equation}
t_m
=
T-\sum_{j=1}^{m-1}r_\star^{\max(j-2,0)},
\qquad
m=1,\ldots,B.
\label{eq:gs-continuous-anchors}
\end{equation}
On the solver grid, apply the guarded order-preserving projection
\begin{equation}
{
t_m^{\mathrm{disc}}
=
\Pi_{\mathrm{grid}}
\!\left(
T-\sum_{j=1}^{m-1}r_\star^{\max(j-2,0)}
\right).
}
\label{eq:gs-anchor-schedule}
\end{equation}

\begin{gsalgorithm}[H]
\caption{GeoShrink}
\label{alg:geoshrink}
\small
\begin{algorithmic}[1]

\Require Model \(V_\theta\), total steps T, solver \(\mathcal D_t\), condition \(c\),
noise \(x_T\), noise schedule \(\{\sigma_t\}_{t=0}^{T}\),
and budget \(B\)

\Ensure Final sample \(x_0\)

\For{\(t=T,T-1,\ldots,1\)}

    \If{
    \(
    \displaystyle
    t\in
    \left\{
    \Pi_{\mathrm{grid}}\!\left(
    T-\sum_{j=1}^{m-1}r_\star^{\max(j-2,0)}
    \right)
    \;\middle|\;
    m=1,\ldots,B,\;
    r_\star\ge1,\;
    \sum_{j=1}^{B-1}r_\star^{\max(j-2,0)}=T-1
    \right\}
    \)
    }

        \State
        \[
        v_t\gets V_\theta(x_t,\sigma_t,c)
 ,\quad
        t^{(3)}\gets t^{(2)},\quad
        t^{(2)}\gets t^{(1)},\quad
        t^{(1)}\gets t
        \]

    \Else

        \State
        \[
      {
        v_t\gets
        v_t^{(1)}+
        \frac{
        v_t^{(1)}-v_t^{(2)}
        }{
        1+
        \left(
        \frac{
        \sigma_t+\sigma_{t-1}-2\sigma_t^{(1)}
        }{
        2\bigl(\sigma_t^{(1)}-\sigma_t^{(2)}\bigr)
        }
        \right)^2
        \sin^2\!\left(
        \angle\!\left(
        v_t^{(1)}-v_t^{(2)},
        v_t^{(2)}-v_t^{(3)}
        \right)
        \right)
        }
        }
        \]

    \EndIf

    \State
    \[
    x_{t-1}\gets
    \mathcal D_t(x_t,v_t;\sigma_t,\sigma_{t-1})
    \]

\EndFor

\State \Return \(x_0\)

\end{algorithmic}
\end{gsalgorithm}

\section{Empirical Analysis}
\label{sec:experiments}

\subsection{Experimental Settings}
\label{sec:experimental-settings}

\paragraph{Models, benchmarks, and baselines.}
We evaluate SD3.5 Medium~\citep{stabilityai2024sd35,esser2024sd3} and
FLUX.1-dev~\citep{blackforestlabs2024flux} on Pick-a-Pic~\citep{kirstain2023pickapic},
and HunyuanVideo~\citep{kong2024hunyuanvideo} on
ChronoMagic-Bench-150~\citep{yuan2024chronomagic} and
VBench~\citep{huang2024vbench}, with additional audiovisual evaluation on
MiniMax-H3~\citep{minimax2026h3}.
Cross-domain experiments use HY-Motion~\citep{wen2025hymotion},
TRELLIS-text-base~\citep{xiang2025trellis} on
T$^3$Bench-300~\citep{he2023t3bench},
Hunyuan3D-2~\citep{zhao2025hunyuan3d2} on GSO-100~\citep{downs2022gso},
TangoFlux~\citep{hung2024tangoflux} on AudioCaps-100~\citep{kim2019audiocaps},
and ACE-Step~\citep{gong2025acestep} on MusicCaps-100~\citep{agostinelli2023musiclm}.
Image and video baselines include FORA~\citep{selvaraju2024fora},
TeaCache~\citep{liu2025teacache}, ToCa~\citep{zou2025toca},
TaylorSeer~\citep{liu2025taylorseer}, SpeCa~\citep{liu2025speca}, and
ResilPhase~\citep{zhao2026resilphase}.

\paragraph{Evaluation metrics.}
We report PSNR~\citep{wang2009mse}, SSIM~\citep{wang2004ssim}, and
LPIPS~\citep{zhang2018lpips} for visual fidelity, together with
ImageReward~\citep{xu2023imagereward}, MTScore~\citep{yuan2024chronomagic},
VBench Score~\citep{huang2024vbench}, CLIP~\citep{radford2021clip}, and
FVD~\citep{unterthiner2019fvd}.
Motion evaluation uses joint, root, velocity, acceleration, and retrieval
metrics~\citep{petrovich2023tmr}; 3D evaluation uses Chamfer
distance~\citep{fan2017pointset}, Hausdorff distance~\citep{huttenlocher1993hausdorff},
IoU, and normal error.
Audio evaluation uses log-spectral distance~\citep{gray1976distance},
multi-resolution STFT distance~\citep{yamamoto2020parallelwavegan},
SI-SDR~\citep{leroux2019sisdr}, and CLAP~\citep{wu2023clap}, with mel-spectrogram
cosine similarity and log-mel L1 for audiovisual evaluation.

\begin{figure}[h]
    \centering
    \includegraphics[width=\linewidth]{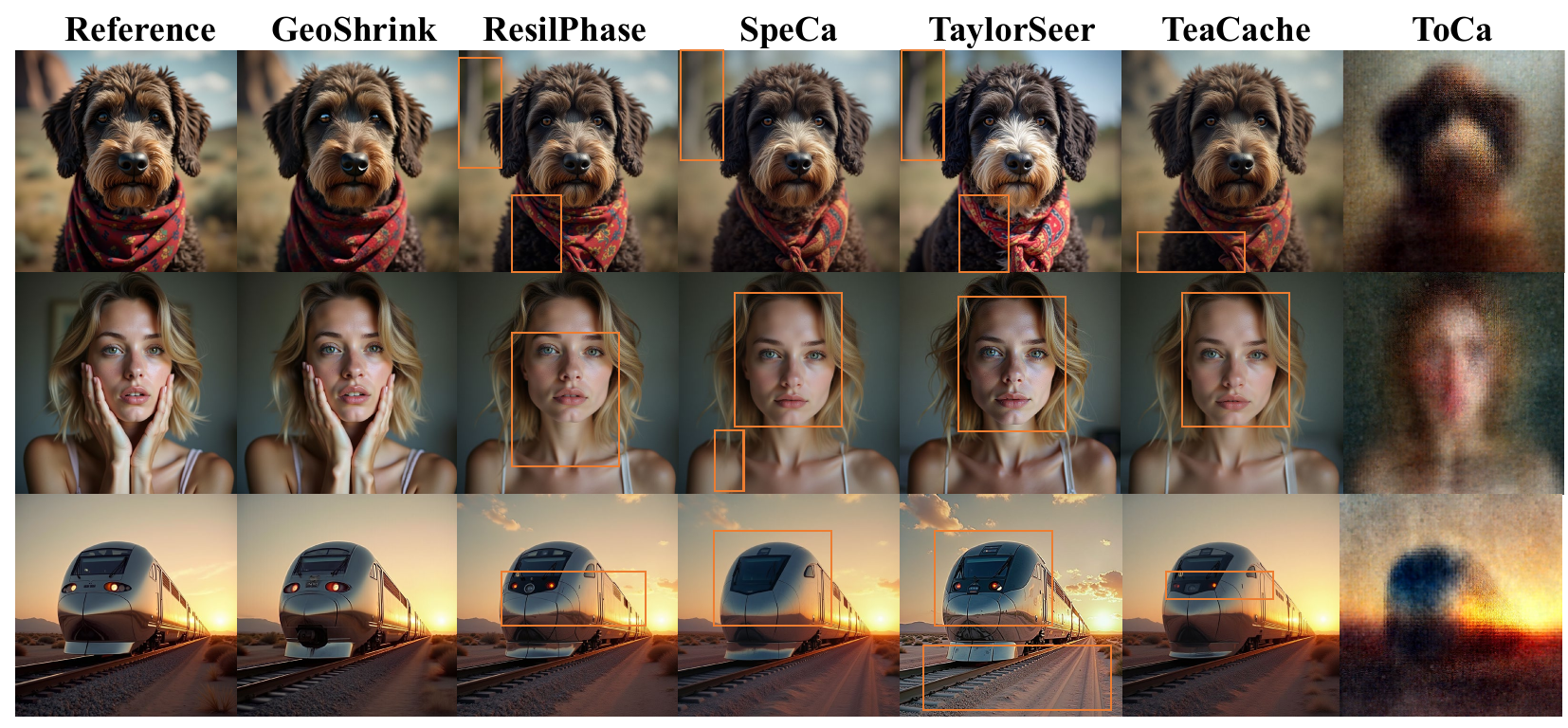}
    \caption{\textbf{Qualitative comparison on text-to-image generation.} GeoShrink better preserves the full-sampling reference than competing acceleration methods at comparable speedups.}
    \label{fig:qual-image}
\end{figure}

\begin{figure}[t]
    \centering
    \includegraphics[width=\linewidth]{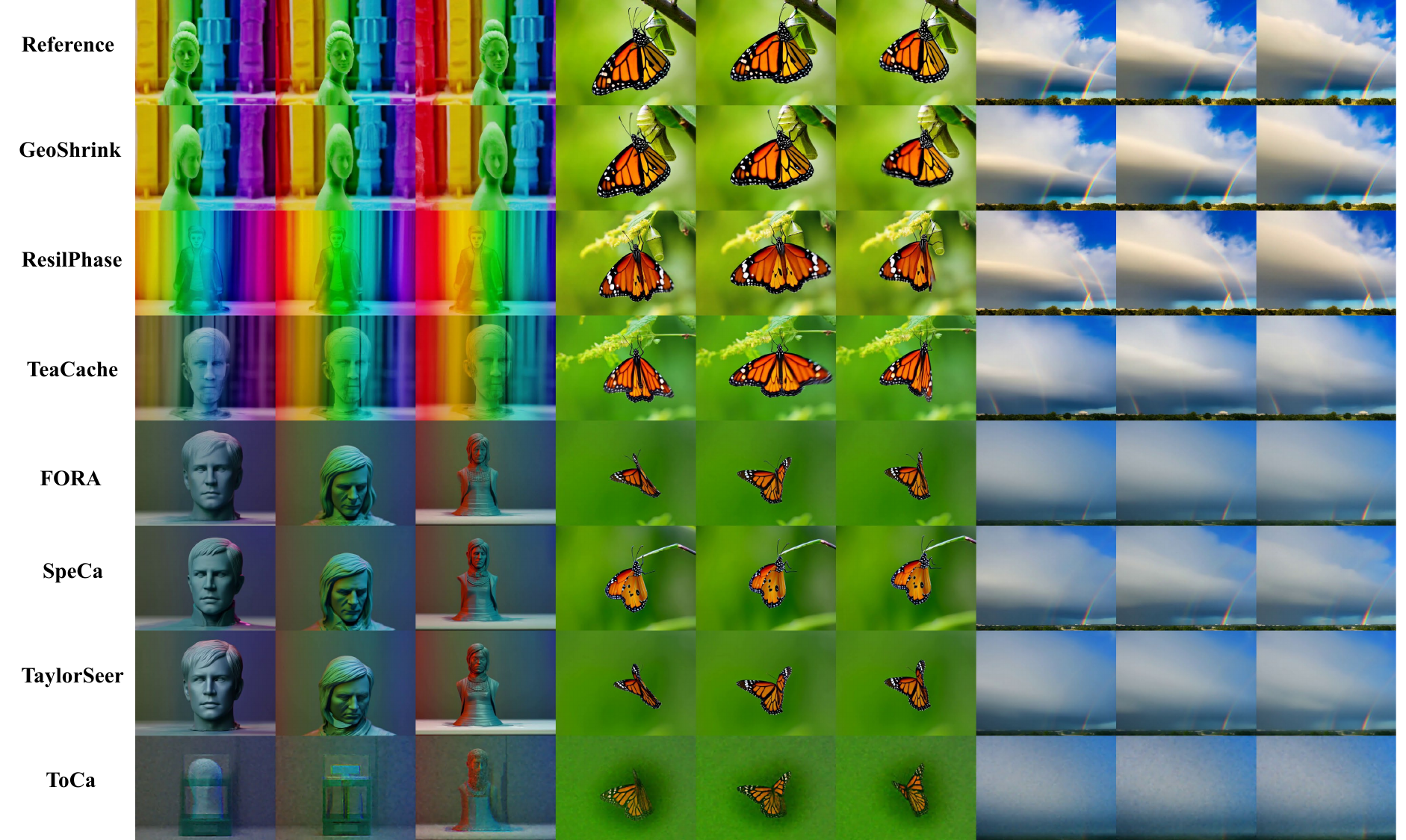}
    \caption{\textbf{Qualitative comparison on text-to-video generation.} GeoShrink preserves spatial details and temporal consistency more faithfully under aggressive acceleration.}
    \label{fig:qual-video}
\end{figure}

\begin{table*}[hbtp]
\centering
\caption{
Quantitative comparison of acceleration methods across image and video generation models.
For the task-specific metric, we report ImageReward on Pick-a-Pic,
MTScore on ChronoMagic-Bench-150, and VBench Score on VBench.
}
\label{tab:main_comparison}
\small
\setlength{\tabcolsep}{5pt}

\resizebox{\textwidth}{!}{
\begin{tabular}{lllrccccc}
\toprule

\textbf{Model}
&
\textbf{Benchmark}
&
\textbf{Method}
&
\textbf{Setting}
&
\textbf{Speedup}
&
\textbf{PSNR} $\uparrow$
&
\textbf{SSIM} $\uparrow$
&
\textbf{LPIPS} $\downarrow$
&
\textbf{Task Score} $\uparrow$
\\

\midrule


\multirow{6}{*}{SD3.5 Medium}
&
\multirow{6}{*}{Pick-a-Pic}
&
SpeCa
& $\tau_0=7,\beta=0.3$
& 2.40$\times$
& 20.759
& 0.7743
& 0.2704
& 0.7873
\\

&
&
TaylorSeer
& $\mathcal{N}=3,O=1$
& 2.49$\times$
& 18.493
& 0.7454
& 0.2918
& 0.7815
\\

&
&
TeaCache
& $l=0.4$
& 2.66$\times$
& 20.872
& 0.7616
& 0.2942
& 0.7671
\\

&
&
ToCa
& $\mathcal{N}=10,R=90\%$
& 2.09$\times$
& 20.431
& 0.7057
& 0.3111
& 0.7662
\\

&
&
ResilPhase
& $\mathcal{N}=3,O=1$
& 2.76$\times$
& 18.483
& 0.7451
& 0.2921
& 0.7813
\\

&
&
\textbf{GeoShrink}
& \textbf{$B=10$}
& \textbf{2.80$\times$}
& \textbf{24.744}
& \textbf{0.8450}
& \textbf{0.2078}
& \textbf{0.7924}
\\

\midrule


\multirow{6}{*}{FLUX.1-dev}
&
\multirow{6}{*}{Pick-a-Pic}
&
SpeCa
& $\tau_0=12,\beta=0.3$
& 4.78$\times$
& 17.4881
& 0.6768
& 0.4481
& 0.9031
\\

&
&
TaylorSeer
& $\mathcal{N}=11,O=2$
& 4.65$\times$
& 15.7641
& 0.6439
& 0.4398
& 0.9181
\\

&
&
TeaCache
& $l=1.4$
& 4.82$\times$
& 18.0757
& 0.6850
& 0.4123
& 0.9207
\\

&
&
ToCa
& $\mathcal{N}=3,R=93\%$
& 4.86$\times$
& 14.8189
& 0.4608
& 0.7701
& 0.7310
\\

&
&
ResilPhase
& $\mathcal{N}=6,O=1$
& 4.97$\times$
& 18.5211
& 0.7083
& 0.3499
& 0.9219
\\

&
&
\textbf{GeoShrink}
& \textbf{$B=10$}
& \textbf{4.98$\times$}
& \textbf{21.6227}
& \textbf{0.7903}
& \textbf{0.2552}
& \textbf{0.9468}
\\

\midrule


\multirow{7}{*}{HunyuanVideo}
&
\multirow{7}{*}{ChronoMagic-150}
&
FORA
& $\mathcal{N}=6$
& 4.87$\times$
& 14.181
& 0.5812
& 0.4924
& 41.4804
\\

&
&
TeaCache
& $l=0.4$
& 4.62$\times$
& 17.620
& 0.6706
& 0.3492
& 41.1334
\\

&
&
ToCa
& $\mathcal{N}=10,R=90\%$
& 4.86$\times$
& 13.335
& 0.5166
& 0.5840
& 42.3887
\\

&
&
SpeCa
& $\tau_0=1.5,\beta=0.2$
& 4.65$\times$
& 14.428
& 0.5867
& 0.4720
& 40.9493
\\

&
&
TaylorSeer
& $\mathcal{N}=7,O=1$
& 4.63$\times$
& 14.303
& 0.5872
& 0.4850
& 41.1623
\\

&
&
ResilPhase
& $\mathcal{N}=6,O=1$
& 4.98$\times$
& 19.881
& 0.6854
& 0.2878
& 42.7770
\\

&
&
\textbf{GeoShrink}
& $B=10$
& \textbf{4.99$\times$}
& \textbf{25.317}
& \textbf{0.8202}
& \textbf{0.1699}
& \textbf{42.9833}
\\

\midrule


\multirow{7}{*}{HunyuanVideo}
&
\multirow{7}{*}{VBench}
&
FORA
& $\mathcal{N}=6$
& 4.87$\times$
& 15.871
& 0.6073
& 0.4584
& 78.70
\\

&
&
TeaCache
& $l=0.4$
& 4.62$\times$
& 17.923
& 0.6547
& 0.3760
& 79.77
\\

&
&
ToCa
& $\mathcal{N}=10,R=90\%$
& 4.86$\times$
& 17.581
& 0.5857
& 0.4501
& 76.37
\\

&
&
SpeCa
& $\tau_0=1.5,\beta=0.2$
& 4.65$\times$
& 16.461
& 0.5883
& 0.4219
& 79.59
\\

&
&
TaylorSeer
& $\mathcal{N}=7,O=1$
& 4.63$\times$
& 15.520
& 0.5641
& 0.4581
& 79.07
\\

&
&
ResilPhase
& $\mathcal{N}=6,O=1$
& 4.98$\times$
& 18.920
& 0.6709
& 0.3341
& 79.78
\\

&
&
\textbf{GeoShrink}
& $B=10$
& \textbf{4.99$\times$}
& \textbf{22.8045}
& \textbf{0.7925}
& \textbf{0.1930}
& \textbf{80.18}
\\

\bottomrule
\end{tabular}
}

\end{table*}

\subsection{Image and Video Generation}
\label{sec:exp-image-video}

As shown in Table~\ref{tab:main_comparison}, GeoShrink improves reference fidelity at comparable speedups.
On SD3.5 Medium and FLUX.1-dev, it improves PSNR over the strongest listed
fidelity baselines by 3.87 dB and 3.10 dB, respectively, while also improving
SSIM and LPIPS. On HunyuanVideo, it reaches \(4.99\times\) speedup and improves
ChronoMagic-Bench-150 PSNR by 5.44 dB.
Figures~\ref{fig:qual-image} and~\ref{fig:qual-video} further show improved visual fidelity,
while Table~\ref{tab:minimax_h3} confirms clear gains over direct step reduction
on MiniMax-H3 at matched NFE.
On MiniMax-H3, improvements extend to both visual and audio metrics across all tested NFE budgets.
These results support the applicability of solver-interface prediction to joint audiovisual generation.

\begin{table}[H]
\centering
\caption{
Quantitative comparison between \textbf{GeoShrink} and direct step reduction on MiniMax-H3.
}
\label{tab:minimax_h3}
\small
\setlength{\tabcolsep}{4.5pt}

\resizebox{\textwidth}{!}{
\begin{tabular}{c l c c c c c c c c}
\toprule

\multirow{2}{*}{\textbf{NFE}}
&
\multirow{2}{*}{\textbf{Method}}
&
\multicolumn{5}{c}{\textbf{Visual}}
&
\multicolumn{3}{c}{\textbf{Audio}}
\\

\cmidrule(lr){3-7}
\cmidrule(lr){8-10}

&
&
\textbf{CLIP} $\uparrow$
&
\textbf{FVD} $\downarrow$
&
\textbf{PSNR} $\uparrow$
&
\textbf{SSIM} $\uparrow$
&
\textbf{LPIPS} $\downarrow$
&
\textbf{Mel Cos} $\uparrow$
&
\textbf{LogMel L1} $\downarrow$
&
\textbf{SI-SDR} $\uparrow$
\\

\midrule

\multirow{2}{*}{10}
& Direct
& 33.7164
& 166.0071
& 17.4910
& 0.6190
& 0.3632
& 0.7730
& 1.1660
& -0.2793
\\[-1pt]

& \textbf{GeoShrink}
& \textbf{33.8711}
& \textbf{164.5004}
& \textbf{24.2122}
& \textbf{0.7883}
& \textbf{0.1678}
& \textbf{0.8528}
& \textbf{0.6998}
& \textbf{1.5133}
\\

\midrule

\multirow{2}{*}{15}
& Direct
& 33.6598
& 166.8756
& 19.2673
& 0.6732
& 0.2914
& 0.8537
& 0.8240
& 3.2408
\\[-1pt]

& \textbf{GeoShrink}
& \textbf{33.7885}
& \textbf{166.1741}
& \textbf{29.1320}
& \textbf{0.8810}
& \textbf{0.0818}
& \textbf{0.9595}
& \textbf{0.3645}
& \textbf{7.6285}
\\

\midrule

\multirow{2}{*}{20}
& Direct
& 33.6435
& 166.9863
& 21.2167
& 0.7250
& 0.2317
& 0.9011
& 0.5850
& 6.3888
\\[-1pt]

& \textbf{GeoShrink}
& \textbf{33.6454}
& \textbf{164.4024}
& \textbf{33.1394}
& \textbf{0.9276}
& \textbf{0.0451}
& \textbf{0.9878}
& \textbf{0.2134}
& \textbf{12.585}
\\

\bottomrule
\end{tabular}
}

\end{table}

\begin{table}[H]
\centering
\caption{
Comparison between \textbf{GeoShrink} and direct step reduction on the HY-Motion.}
\label{tab:hymotion_direct_vs_geo}
\small
\setlength{\tabcolsep}{4.5pt}

\resizebox{\textwidth}{!}{
\begin{tabular}{c l c c c c c c c}
\toprule
\textbf{NFE}
& \textbf{Method}
& \textbf{RA-MPJPE} $\downarrow$
& \textbf{Global MPJPE} $\downarrow$
& \textbf{RootErr} $\downarrow$
& \textbf{VelErr} $\downarrow$
& \textbf{AccErr} $\downarrow$
& \textbf{uTMR R@3} $\uparrow$
& \textbf{Matching} $\downarrow$
\\
\midrule

\multirow{2}{*}{10}
& Direct
& 120.618
& 501.158
& 465.109
& 7.325
& 2.484
& 0.6260
& 49.8151
\\
& \textbf{GeoShrink}
& \textbf{5.303}
& \textbf{26.009}
& \textbf{25.024}
& \textbf{0.694}
& \textbf{0.301}
& \textbf{0.7280}
& \textbf{46.8595}
\\
\midrule

\multirow{2}{*}{20}
& Direct
& 81.194
& 360.800
& 335.515
& 5.822
& 1.948
& 0.7105
& 47.9181
\\
& \textbf{GeoShrink}
& \textbf{1.252}
& \textbf{6.004}
& \textbf{5.786}
& \textbf{0.199}
& \textbf{0.102}
& \textbf{0.7320}
& \textbf{46.7981}
\\
\midrule

\multirow{2}{*}{30}
& Direct
& 51.350
& 194.661
& 177.461
& 4.035
& 1.363
& 0.7175
& 47.2669
\\
& \textbf{GeoShrink}
& \textbf{0.394}
& \textbf{2.337}
& \textbf{2.279}
& \textbf{0.093}
& \textbf{0.063}
& \textbf{0.7300}
& \textbf{46.8036}
\\
\midrule

\multirow{2}{*}{40}
& Direct
& 25.686
& 86.252
& 77.388
& 2.219
& 0.766
& 0.7155
& 46.9870
\\
& \textbf{GeoShrink}
& \textbf{0.100}
& \textbf{0.799}
& \textbf{0.789}
& \textbf{0.046}
& \textbf{0.047}
& \textbf{0.7300}
& \textbf{46.8060}
\\

\bottomrule
\end{tabular}
}
\end{table}

\subsection{Human Motion Generation}
\label{sec:exp-motion}

As shown in Table~\ref{tab:hymotion_direct_vs_geo} and
Figure~\ref{fig:qual-motion}, the benefit of retaining the dense solver trajectory
is particularly pronounced for human motion.
Across all tested NFE budgets, GeoShrink sharply reduces geometric and kinematic
errors over direct step reduction while improving retrieval-based semantic metrics.
The large gap at matched NFE suggests that motion trajectories are especially
sensitive to removing intermediate solver updates.

\begin{figure}[t]
    \centering

    \includegraphics[width=\linewidth]{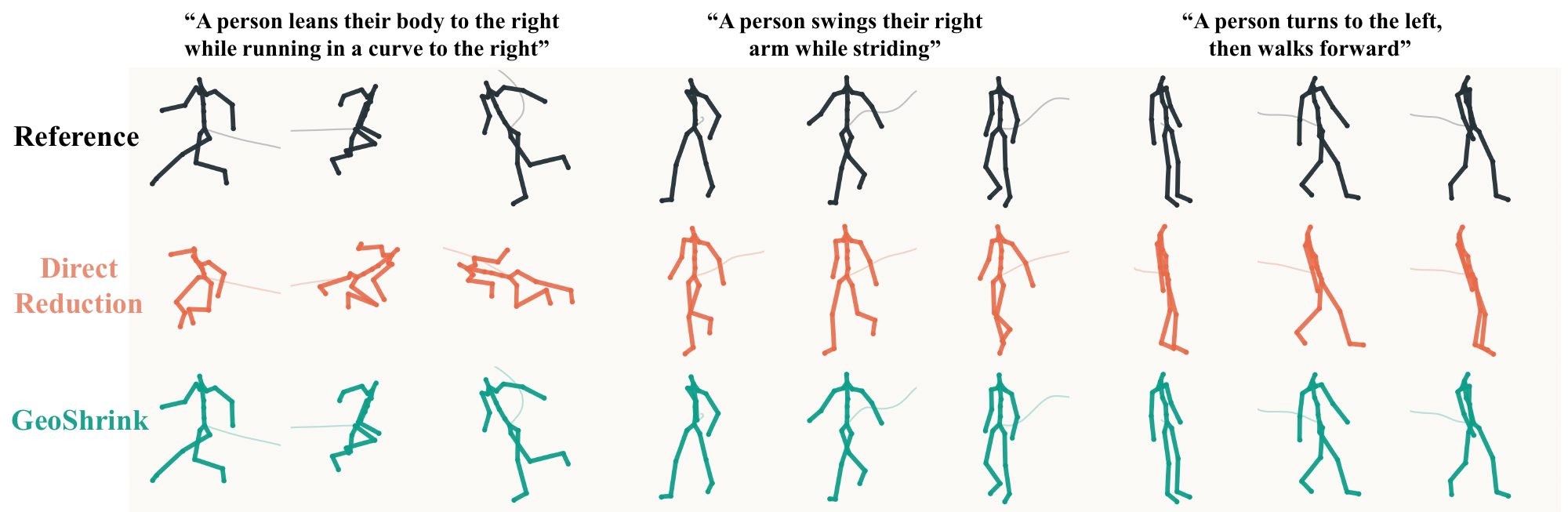}
    \caption{\textbf{Motion qualitative comparison.} GeoShrink better preserves full-sampling trajectories.}
    \label{fig:qual-motion}

    \vspace{0mm}

    \includegraphics[width=\linewidth]{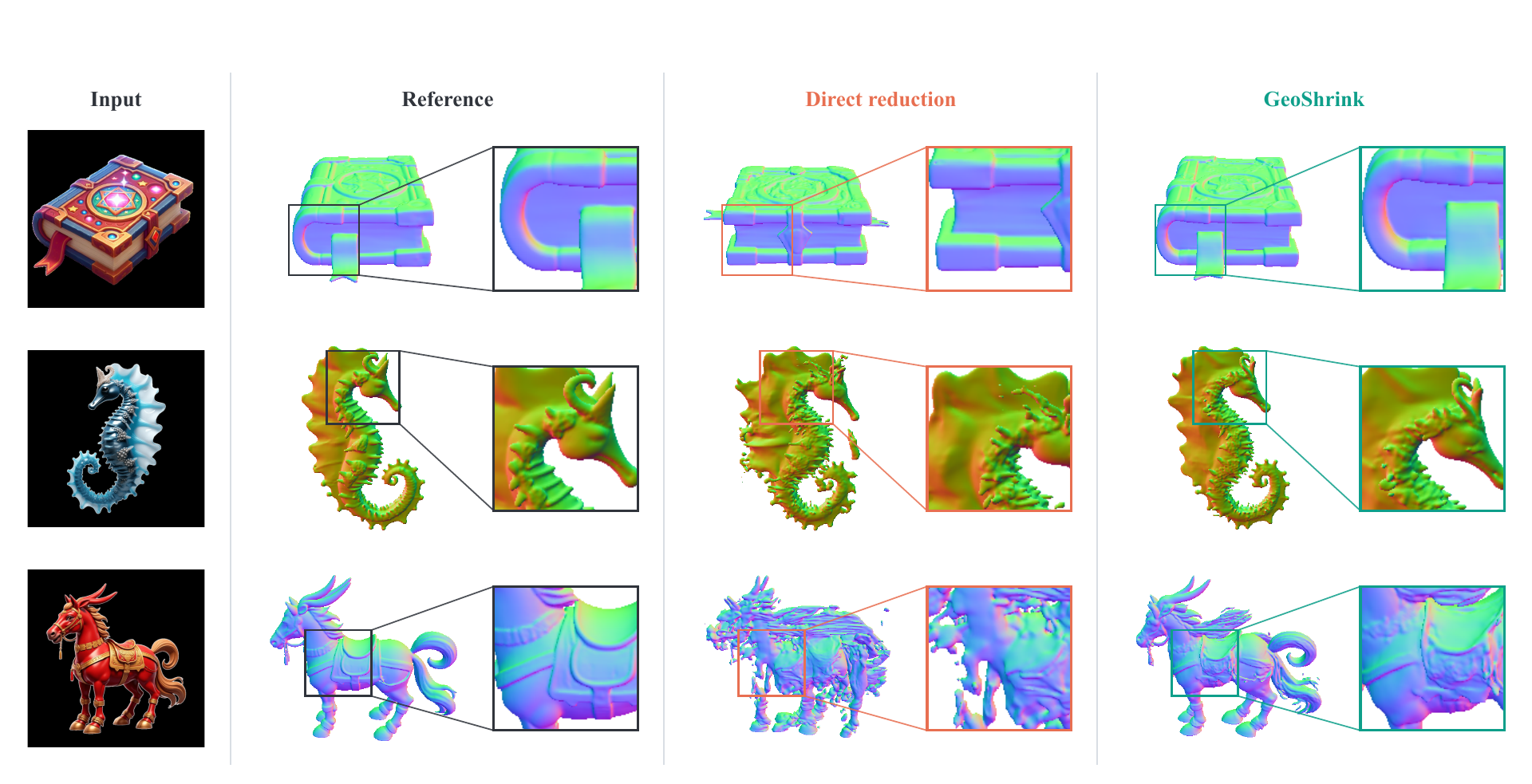}
    \caption{\textbf{3D qualitative comparison.} Left to right: input, reference (\textcolor{black}{black}), direct reduction (\textcolor{boxred}{red}), and GeoShrink (\textcolor{boxgreen}{green}). GeoShrink better preserves geometry.}
    \label{fig:qual-3d}
    \vspace{-2mm}
\end{figure}

\begin{figure}[t]
    \centering
    \includegraphics[width=\linewidth]{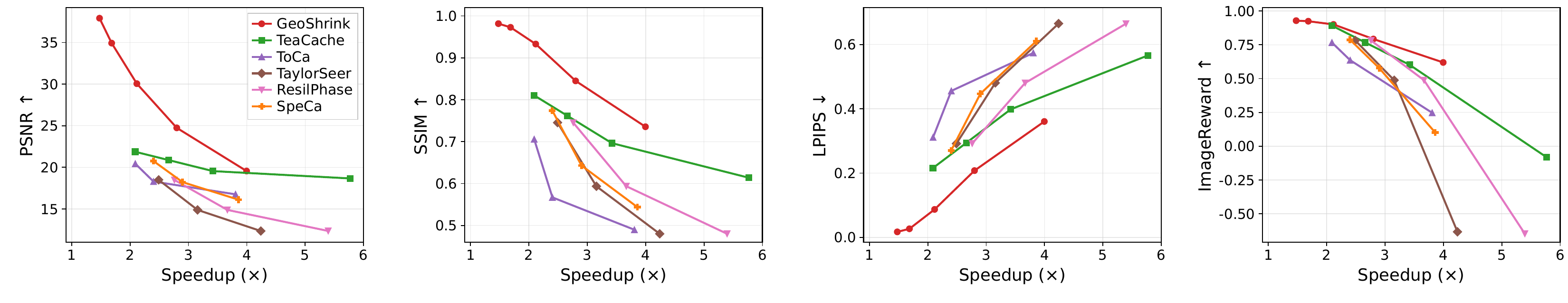}
    \caption{\textbf{Speed-quality trade-off.} GeoShrink yields a better Pareto frontier on SD3.5 Medium.}
    \label{fig:pareto}
\end{figure}

\begin{table}[h]
\centering
\caption{Comparison between \textbf{GeoShrink} implementations and direct step reduction for 3D generation. We evaluate TRELLIS-text-base on T$^3$Bench-300 and Hunyuan3D-2 on GSO-100.}
\label{tab:3d_direct_vs_geo}
\small
\setlength{\tabcolsep}{4pt}
\begin{adjustbox}{max width=\textwidth}
\begin{tabular}{c l rrrr @ {\hspace{14pt}} c l rrrr}
\toprule
\multicolumn{6}{c}{\textbf{TRELLIS-text-base / T$^3$Bench-300}}
& \multicolumn{6}{c}{\textbf{Hunyuan3D-2 / GSO-100}} \\
\cmidrule(lr){1-6}
\cmidrule(lr){7-12}
\textbf{NFE} & \textbf{Method}
& \textbf{CLIP} $\uparrow$
& \textbf{Center CD}$^2$ $\downarrow$ & \textbf{LPIPS} $\downarrow$
& \textbf{Hausdorff} $\downarrow$
& \textbf{NFE} & \textbf{Method}
& \textbf{CLIP} $\uparrow$
& \textbf{Surface CD}$^2$ $\downarrow$ & \textbf{IoU} $\uparrow$
& \textbf{Normal} ($^\circ$) $\downarrow$ \\
\midrule
\multirow{2}{*}{10} & Direct & $26.6521$ & $2.007\times 10^{-6}$ & $0.009377$ & $0.013083$ & \multirow{2}{*}{10} & Direct & $74.3904$ & $5.411\times 10^{-3}$ & $0.89199$ & $15.3606$ \\
 & \textbf{GeoShrink} & $\mathbf{26.7623}$ & $\mathbf{1.005\times 10^{-6}}$ & $\mathbf{0.004232}$ & $\mathbf{0.010800}$ &  & \textbf{GeoShrink} & $\mathbf{74.4081}$ & $\mathbf{5.270\times 10^{-3}}$ & $\mathbf{0.89790}$ & $\mathbf{15.0119}$ \\
\midrule
\multirow{2}{*}{12} & Direct & $\mathbf{26.7772}$ & $1.529\times 10^{-6}$ & $0.006723$ & $0.012179$ & \multirow{2}{*}{15} & Direct & $75.1844$ & $\mathbf{6.383\times 10^{-4}}$ & $\mathbf{0.95457}$ & $\mathbf{9.4850}$ \\
 & \textbf{GeoShrink} & $26.7711$ & $\mathbf{4.042\times 10^{-7}}$ & $\mathbf{0.001528}$ & $\mathbf{0.008603}$ &  & \textbf{GeoShrink} & $\mathbf{75.3055}$ & $6.704\times 10^{-4}$ & $0.95319$ & $9.5296$ \\
\midrule
\multirow{2}{*}{15} & Direct & $26.6790$ & $1.176\times 10^{-6}$ & $0.004662$ & $0.011269$ & \multirow{2}{*}{20} & Direct & $75.2853$ & $\mathbf{4.383\times 10^{-4}}$ & $0.96723$ & $7.6636$ \\
 & \textbf{GeoShrink} & $\mathbf{26.7704}$ & $\mathbf{3.845\times 10^{-7}}$ & $\mathbf{0.001484}$ & $\mathbf{0.008399}$ &  & \textbf{GeoShrink} & $\mathbf{75.3562}$ & $4.505\times 10^{-4}$ & $\mathbf{0.96818}$ & $\mathbf{7.4691}$ \\
\midrule
\multirow{2}{*}{18} & Direct & $26.7687$ & $5.040\times 10^{-7}$ & $0.001729$ & $0.008954$ & \multirow{2}{*}{25} & Direct & $75.4099$ & $3.750\times 10^{-4}$ & $0.97453$ & $6.5459$ \\
 & \textbf{GeoShrink} & $\mathbf{26.7743}$ & $\mathbf{1.710\times 10^{-7}}$ & $\mathbf{0.000603}$ & $\mathbf{0.006766}$ &  & \textbf{GeoShrink} & $\mathbf{75.4237}$ & $\mathbf{1.796\times 10^{-4}}$ & $\mathbf{0.97631}$ & $\mathbf{6.1900}$ \\
\midrule
\multirow{2}{*}{20} & Direct & $\mathbf{26.8150}$ & $4.745\times 10^{-7}$ & $0.001613$ & $0.008815$ & \multirow{2}{*}{30} & Direct & $75.4142$ & $1.380\times 10^{-4}$ & $0.98147$ & $5.6879$ \\
 & \textbf{GeoShrink} & $26.7753$ & $\mathbf{1.593\times 10^{-7}}$ & $\mathbf{0.000584}$ & $\mathbf{0.006574}$ &  & \textbf{GeoShrink} & $\mathbf{75.4569}$ & $\mathbf{8.061\times 10^{-5}}$ & $\mathbf{0.98253}$ & $\mathbf{5.5189}$ \\
\bottomrule
\end{tabular}
\end{adjustbox}

\par\smallskip
\begin{minipage}{\textwidth}
\end{minipage}
\vspace{-3mm}
\end{table}

\subsection{3D Generation}
\label{sec:exp-3d}
Table~\ref{tab:3d_direct_vs_geo} and Figure~\ref{fig:qual-3d} show that
GeoShrink generally improves reference fidelity over direct step reduction,
especially on TRELLIS, where Center Chamfer Distance, LPIPS, and Hausdorff
distance are consistently reduced.
Results on Hunyuan3D-2 are more heterogeneous, with clear gains at several
budgets and small reversals on individual geometry metrics.

\begin{table}[H]
\centering
\caption{Comparison between \textbf{GeoShrink} and direct step reduction for audio generation at matched NFE budgets. We evaluate TangoFlux on AudioCaps-100 and ACE-Step on MusicCaps-100.}
\label{tab:audio_direct_vs_geo}
\small

\begin{adjustbox}{max width=\textwidth}
\begin{tabular}{c l cccc cccc}
\toprule

\multirow{2}{*}{\textbf{NFE}}
&
\multirow{2}{*}{\textbf{Method}}
&
\multicolumn{4}{c}{\textbf{TangoFlux / AudioCaps-100}}
&
\multicolumn{4}{c}{\textbf{ACE-Step / MusicCaps-100}}
\\

\cmidrule(lr){3-6}
\cmidrule(lr){7-10}

&
&
\textbf{LSD} $\downarrow$
&
\textbf{MS-STFT} $\downarrow$
&
\textbf{SI-SDR} $\uparrow$
&
\textbf{CLAP} $\uparrow$
&
\textbf{LSD} $\downarrow$
&
\textbf{MS-STFT} $\downarrow$
&
\textbf{SI-SDR} $\uparrow$
&
\textbf{CLAP} $\uparrow$
\\

\midrule

\multirow{2}{*}{15}
& Direct
& 9.594 & 0.458 & 3.83 & 0.6857
& 8.524 & 0.454 & \textbf{3.66} & 0.4906
\\
& \textbf{GeoShrink}
& \textbf{6.233} & \textbf{0.228} & \textbf{11.06} & \textbf{0.6886}
& \textbf{8.156} & \textbf{0.453} & 3.58 & \textbf{0.4947}
\\

\midrule

\multirow{2}{*}{20}
& Direct
& 8.006 & 0.332 & 6.92 & 0.6902
& 7.989 & 0.395 & 5.45 & 0.4902
\\
& \textbf{GeoShrink}
& \textbf{4.852} & \textbf{0.090} & \textbf{18.52} & \textbf{0.6914}
& \textbf{7.596} & \textbf{0.340} & \textbf{6.87} & \textbf{0.4955}
\\

\midrule

\multirow{2}{*}{25}
& Direct
& 6.744 & 0.245 & 10.10 & 0.6922
& 7.513 & \textbf{0.309} & \textbf{8.25} & 0.4887
\\
& \textbf{GeoShrink}
& \textbf{4.295} & \textbf{0.045} & \textbf{24.39} & \textbf{0.6928}
& \textbf{7.486} & 0.322 & 7.47 & \textbf{0.4923}
\\

\midrule

\multirow{2}{*}{30}
& Direct
& 6.184 & 0.184 & 13.48 & 0.6913
& 7.333 & 0.285 & 9.18 & 0.4895
\\
& \textbf{GeoShrink}
& \textbf{3.922} & \textbf{0.016} & \textbf{33.55} & \textbf{0.6931}
& \textbf{6.945} & \textbf{0.229} & \textbf{11.30} & \textbf{0.4922}
\\

\bottomrule
\end{tabular}
\end{adjustbox}

\end{table}

\subsection{Audio and Music Generation}
\label{sec:exp-audio}
As shown in Table~\ref{tab:audio_direct_vs_geo}, GeoShrink transfers effectively
to audio generation.
On TangoFlux, it consistently improves spectral and waveform fidelity across
all tested NFE budgets, with especially large gains in SI-SDR.
Results on ACE-Step are more mixed: GeoShrink generally improves LSD and CLAP
but does not dominate every spectral metric at every budget, indicating that
music generation can be more metric-dependent.

\subsection{Ablation Study}

{
\setlength{\intextsep}{3pt}   

\begin{table}[h]
\centering

\begin{minipage}[t]{0.44\textwidth}
\vspace{0pt}

\textbf{Component ablation.}
Table~\ref{tab:ablation} shows that geometric adaptation and shrink
interpolation improve quality on SD3.5 Medium at $4\times$ acceleration.
Both components provide individual gains, while their combination
performs best across all metrics, confirming their complementary effects.

\end{minipage}
\hfill
\begin{minipage}[t]{0.54\textwidth}
\vspace{0pt}
\centering

\caption{Ablation study of geometric adaptation and innovation shrinkage using SD3.5 Medium.}
\label{tab:ablation}

\vspace{-1mm}

\small
\setlength{\tabcolsep}{5pt}
\renewcommand{\arraystretch}{0.95}

\resizebox{\linewidth}{!}{
\begin{tabular}{lcccc}
\toprule
\textbf{Method}
&
\textbf{PSNR} $\uparrow$
&
\textbf{SSIM} $\uparrow$
&
\textbf{LPIPS} $\downarrow$
&
\textbf{ImageReward} $\uparrow$
\\
\midrule

Fixed-Linear
& 9.882
& 0.404832
& 0.702306
& -1.2556
\\

Geo-Linear
& 15.457
& 0.621274
& 0.414758
& 0.278629
\\

Fixed-Shrink
& 19.556
& 0.692327
& 0.388319
& 0.559383
\\

\textbf{GeoShrink}
& \textbf{19.851}
& \textbf{0.735519}
& \textbf{0.360551}
& \textbf{0.619176}
\\

\bottomrule
\end{tabular}
}

\end{minipage}

\end{table}
}

\vspace{-1mm}
\noindent
\textbf{Budget robustness.}
As shown in Figure~\ref{fig:pareto}, GeoShrink maintains a stronger
speed-quality trade-off across the tested acceleration range.
At comparable speedups, it achieves higher PSNR and SSIM and lower LPIPS,
while maintaining competitive ImageReward scores.
Its fidelity advantage remains evident as acceleration increases, suggesting
greater robustness to reduced computation budgets.
These results show that the gains extend across multiple metrics and operating
points, allowing GeoShrink to accommodate different inference budgets while
preserving output quality.

\section{Conclusion}

We introduced \textbf{GeoShrink}, a training-free acceleration method that preserves the original solver grid while sparsifying expensive model evaluations. GeoShrink predicts skipped solver-facing outputs through geometry-aware innovation shrinkage and places exact queries using a geometric anchor schedule. Experiments across image, video, motion, audio, music, and 3D generation show that GeoShrink consistently improves fidelity under aggressive acceleration and yields a stronger speed-quality trade-off. Overall, GeoShrink demonstrates that preserving the solver grid can be more effective than directly shortening it.

\bibliographystyle{plainnat}
\bibliography{GeoShrink_references}

\appendix

\section{Related Works}
\label{app}

Diffusion acceleration has been explored through improved numerical solvers,
temporal reuse, computation reduction, and low-overhead test-time refinement.

\textbf{Fast numerical sampling.}
Early work reduces the number of denoising evaluations by designing more accurate
discretizations of diffusion ODEs or SDEs.
DDIM~\citep{song2021ddim} enables non-Markovian sampling with fewer denoising
steps, while PNDM~\citep{liu2022pndm}, DEIS~\citep{zhang2023deis},
DPM-Solver~\citep{lu2022dpmsolver}, DPM-Solver++~\citep{lu2022dpmsolverpp},
UniPC~\citep{zhao2023unipc}, DPM-Solver-v3~\citep{zheng2023dpmsolverv3},
and SA-Solver~\citep{xue2023sasolver} improve low-budget sampling through
higher-order or structure-aware integration.
EDM~\citep{karras2022edm} jointly studies parameterization and discretization,
while Align Your Steps~\citep{sabour2024ays} directly optimizes the sampling
schedule for a fixed model and solver.

\textbf{Temporal reuse and feature caching.}
A large body of work exploits temporal redundancy in intermediate diffusion
representations.
DeepCache~\citep{ma2024deepcache}, Faster Diffusion~\citep{li2023fasterdiffusion},
and FRDiff~\citep{so2023frdiff} reuse slowly varying U-Net features across
timesteps.
For diffusion transformers, Learning-to-Cache~\citep{ma2024l2c},
HarmoniCa~\citep{huang2024harmonica}, $\Delta$-DiT~\citep{chen2024deltadit},
FORA~\citep{selvaraju2024fora}, Pyramid Attention Broadcast~\citep{zhao2024pab},
FasterCache~\citep{lv2025fastercache}, TeaCache~\citep{liu2025teacache},
ToCa~\citep{zou2025toca}, AdaCache~\citep{kahatapitiya2024adacache},
SmoothCache~\citep{liu2024smoothcache}, DuCa~\citep{zou2024duca},
EasyCache~\citep{zhou2025easycache}, and MagCache~\citep{ma2025magcache}
develop increasingly adaptive policies for deciding what and when to reuse.
More recently, TaylorSeer~\citep{liu2025taylorseer},
SpeCa~\citep{liu2025speca}, ResilPhase~\citep{zhao2026resilphase},
and LESA~\citep{cai2026lesa} go beyond direct caching by forecasting future
intermediate representations from historical features.

\textbf{Efficient test-time refinement and incremental adaptation.}
A complementary line of training-free work improves generation through
lightweight corrections applied during inference, progressively refining the
trajectory, guidance signal, initialization, or reusable experience without
updating model parameters.
POLARIS~\citep{chen2025polaris} performs step-wise adaptation of the guidance
scale to suppress accumulated inversion error with negligible additional
overhead, while Guided Path Sampling~\citep{li2026gps} stabilizes iterative
denoising--inversion refinement through path-aware guidance.
$Z^2$-Sampling~\citep{li2026z2sampling} further eliminates the extra evaluations
of explicit zigzag refinement by algebraically collapsing the trajectory and
reusing a temporal semantic surrogate, recovering the standard sampling cost.
At a finer spatial level, Delta Score~\citep{li2026deltascore} adapts guidance
locally according to the conditional--unconditional prediction difference,
providing region-dependent semantic correction without retraining.
Oracle Noise~\citep{li2026oraclenoise} instead performs fast constrained
optimization of the initial latent to incrementally improve semantic alignment
before denoising.
Beyond a single sampling trajectory, MemoGen~\citep{chen2026memogen} extends
test-time refinement across generation episodes by accumulating successful
strategies and failure feedback as reusable experience memory.
Together, these methods demonstrate that substantial improvements can be
obtained through lightweight incremental interventions at inference time.
GeoShrink follows the same training-free efficiency principle but targets a
different source of redundancy: rather than refining guidance, inversion,
initialization, or external memory, it sparsifies expensive model evaluations
while preserving the original solver grid and predicts only the missing
solver-facing outputs.

\section{Implementation details}
All experiments are conducted on NVIDIA A100 GPUs, with each experiment repeated using five random seeds. We report sampling latency, excluding encoding and decoding. Speedups are computed relative to the corresponding full-NFE reference using the same timing protocol. For each example and seed, all methods and the corresponding reference use identical inputs, conditioning, and initial noise, with all other inference settings held fixed except for the acceleration configuration.

We evaluate SD3.5 Medium and FLUX.1-dev on 100 prompts from the Pick-a-Pic test set, using 28-step and 50-step reference grids, respectively. Both models use an image resolution of \(1024\times1024\), with guidance scales of 4.5 and 3.5, respectively. HunyuanVideo is evaluated on ChronoMagic-Bench-150 and the full VBench prompt set using a 50-step reference grid. It produces 65-frame videos at \(640\times480\) resolution and 24 fps, with a guidance scale of 5.0.

For MiniMax-H3, we evaluate all 150 prompts from ChronoMagic-Bench-150 using a 30-step reference grid and evaluation budgets of \(B\in\{10,15,20\}\). Videos contain 124 frames at \(960\times544\) resolution and 24 fps. Inference uses BF16 precision and automatic offloading with a 12 GB memory reserve.

For HY-Motion-1.0, we select 100 examples through stratified sampling from the official 1,983-example SSAE set. We generate five-second motions with the Euler solver, a guidance scale of 5.0, and a 50-step reference grid, evaluating budgets of \(B\in\{10,20,30,40\}\). Prompt rewriting and duration estimation are disabled.

For TRELLIS-text-base, we evaluate all 300 examples from T³Bench, comprising 100 examples per category. The sparse-structure stage remains fixed at 25 steps, while the SLAT stage uses a 25-step reference grid. We set the guidance scale to 7.5, the guidance interval to \([0.5,0.95]\), and the time-rescaling parameter to 3.0. For Hunyuan3D-2, we evaluate shape generation on 100 objects selected from GSO by a fixed hash-based rule. Inference uses FP16 precision, a 50-step reference grid, a guidance scale of 5.0, a latent shape of \(3072\times64\), and a mesh extraction resolution of 256.

For TangoFlux and ACE-Step, we select 100 examples at fixed index intervals from the AudioCaps test split and MusicCaps, respectively. Both models generate ten-second audio using a 40-step reference grid. TangoFlux uses a sampling rate of 44.1 kHz and a guidance scale of 4.5. ACE-Step uses 48 kHz, BF16 precision, a guidance scale of 7.0, and a shift parameter of 3.0, with empty lyrics in text-to-music mode.

\section{Proofs and Additional Analysis}
\label{app:gs-proofs}

In the appendix, \(q\) indexes the latest exact query, \(a_q\) is its solver stage, and \(n\) is the current skipped stage. Write \(v_n^{(i)}=V_\theta(\widetilde x_{a_{q-i+1}},\sigma_n^{(i)},c)\) and \(\sigma_n^{(i)}=\sigma_{a_{q-i+1}}\) for \(i\in\{1,2,3\}\). The anchor-indexed innovations are \(\Delta v_q=v_n^{(1)}-v_n^{(2)}\) and \(\Delta v_q^{-}=v_n^{(2)}-v_n^{(3)}\); \(\theta_q\in[0,\pi/2]\) is the acute angle between their lines, and \(\omega_{q,n}\) denotes the corresponding retention coefficient.

\subsection{Directional Prediction Loss and the Benefit of Shrinkage}
\label{app:gs-loss}

\begin{gsproposition}[Directional prediction loss]
\label{prop:gs-loss}
For a skipped stage, define the unavailable target innovation at the same
current state by
\[
  I_{q,n}=V_\theta(\widetilde x_n,\sigma_n,c)-v_n^{(1)}.
\]
Suppose \(A_q=\gsnorm{{\Delta v_q}}^2>0\), and define
\begin{equation}
  c_{q,n}=\frac{\gsip{{\Delta v_q}}{I_{q,n}}}{A_q}.
  \label{eq:gs-oracle-coefficient}
\end{equation}
For every real \(\omega\),
\begin{equation}
  {
  \gsnorm{\omega{\Delta v_q}-I_{q,n}}^2
  =\gsnorm{P_{{\Delta v_q}}^{\perp}I_{q,n}}^2
   +A_q(\omega-c_{q,n})^2.}
  \label{eq:gs-loss-decomposition}
\end{equation}
Consequently, the optimal coefficient in \eqref{eq:gs-family} is
\(\omega^\star_{q,n}=\gsclip(c_{q,n})\). For \(0<\omega<1\), the prediction strictly
improves on both \(\omega=0\) and \(\omega=1\) if and only if
\begin{equation}
  {\frac{\omega}{2}<c_{q,n}<\frac{1+\omega}{2}.}
  \label{eq:gs-gain-region}
\end{equation}
\end{gsproposition}

\begin{proof}[Proof of Proposition~\ref{prop:gs-loss}]
Fix an anchor-target pair and abbreviate
\({\Delta v}={\Delta v_q}\), \(I=I_{q,n}\), \(A=\gsnorm{{\Delta v}}^2>0\), and
\(c=\gsip{{\Delta v}}{I}/A\). Orthogonal decomposition gives
\[
  I=c{\Delta v}+P_{\Delta v}^\perp I.
\]
Therefore
\[
  \omega{\Delta v}-I=(\omega-c){\Delta v}-P_{\Delta v}^\perp I.
\]
The two terms are orthogonal, so the Pythagorean theorem gives
\eqref{eq:gs-loss-decomposition}. The only term depending on \(\omega\)
is \(A(\omega-c)^2\). Its minimizer over \([0,1]\) is the Euclidean projection
\(\gsclip(c)\).

Let \(\mathcal L(\omega)=\gsnorm{\omega{\Delta v}-I}^2\). Direct subtraction yields
\begin{align}
  \mathcal L(\omega)-\mathcal L(0)
    &=A\,\omega(\omega-2c),\label{eq:gs-gain-vs-hold}\\
  \mathcal L(\omega)-\mathcal L(1)
    &=A\,(\omega-1)(\omega+1-2c).\label{eq:gs-gain-vs-full}
\end{align}
For \(0<\omega<1\), the first difference is negative precisely when
\(c>\omega/2\), and the second is negative precisely when
\(c<(1+\omega)/2\). Their intersection is \eqref{eq:gs-gain-region}.
\end{proof}

If \({\Delta v_q}=0\), all members of \eqref{eq:gs-family} coincide
with \(v_n^{(1)}\). No coefficient needs to be identified, and
\eqref{eq:gs-oracle-coefficient} is left undefined. The implemented
predictor automatically returns \(v_n^{(1)}\) in this case.

\subsection{Conditional Calibration of a Scalar Exposure}
\label{app:gs-calibration}

The scalar \(R\) need not contain all useful information in the query
history. The following statement specifies the optimum only within the
class of retention rules that are functions of \(R\).

\begin{gsproposition}[Best retention rule based on one scalar]
\label{prop:gs-calibration}
Consider a distribution of anchor-target pairs, with
\(\gsE\gsnorm{{\Delta v}}^2<\infty\) and
\(\gsE\gsnorm{I}^2<\infty\). Let \(R\) be any measurable scalar statistic.
Where
\(a(r)=\gsE[\gsnorm{{\Delta v}}^2\mid R=r]>0\), the minimizer over
measurable \(f:\gsR\to[0,1]\) of
\(\gsE\gsnorm{f(R){\Delta v}-I}^2\) is, almost surely,
\begin{equation}
  {
  f^\star(r)=\gsclip\left(
  \frac{\gsE[\gsip{{\Delta v}}{I}\mid R=r]}
       {\gsE[\gsnorm{{\Delta v}}^2\mid R=r]}
  \right).}
  \label{eq:gs-conditional-optimum}
\end{equation}
On a conditional set with \(a(r)=0\), every coefficient has the same loss.
\end{gsproposition}

\begin{proof}
Let
\[
 b(r)=\gsE[\gsip{{\Delta v}}{I}\mid R=r],
 \qquad
 d(r)=\gsE[\gsnorm I^2\mid R=r].
\]
For a deterministic coefficient \(\omega\), expand the squared norm before
taking the conditional expectation:
\begin{align*}
 \gsnorm{\omega{\Delta v}-I}^2
 &=\gsip{\omega{\Delta v}-I}{\omega{\Delta v}-I}\\
 &=\omega^2\gsnorm{\Delta v}^2-2\omega\gsip{\Delta v} I+\gsnorm I^2.
\end{align*}
Linearity of conditional expectation gives
\begin{equation}
 \gsE[\gsnorm{\omega{\Delta v}-I}^2\mid R=r]
 =a(r)\omega^2-2b(r)\omega+d(r).
 \label{eq:gs-conditional-quadratic}
\end{equation}
If \(a(r)>0\), complete the square:
\begin{align*}
 a(r)\omega^2-2b(r)\omega+d(r)
 &=a(r)\left[\omega^2-2\frac{b(r)}{a(r)}\omega\right]+d(r)\\
 &=a(r)\left(\omega-\frac{b(r)}{a(r)}\right)^2
   +d(r)-\frac{b(r)^2}{a(r)}.
\end{align*}
The last two terms do not depend on \(\omega\). The closest point in the closed
interval \([0,1]\) to the unconstrained minimizer \(b(r)/a(r)\) is its
Euclidean projection \(\gsclip(b(r)/a(r))\), which proves
\eqref{eq:gs-conditional-optimum} on \(\{a(r)>0\}\).

To pass from conditional to global risk, let \(f\) be any measurable rule
with values in \([0,1]\), and let \(f^\star\) be the conditional minimizer.
The pointwise conditional inequality gives, almost surely,
\[
 \gsE[\gsnorm{f^\star(R){\Delta v}-I}^2\mid R]
 \le
 \gsE[\gsnorm{f(R){\Delta v}-I}^2\mid R].
\]
Taking expectations of both sides and using the tower property gives
\begin{align*}
 \gsE\gsnorm{f^\star(R){\Delta v}-I}^2
 &=\gsE\!\left[
   \gsE[\gsnorm{f^\star(R){\Delta v}-I}^2\mid R]\right]\\
 &\le\gsE\!\left[
   \gsE[\gsnorm{f(R){\Delta v}-I}^2\mid R]\right]\\
 &=\gsE\gsnorm{f(R){\Delta v}-I}^2.
\end{align*}
If \(a(r)=0\), then the nonnegative random variable
\(\gsnorm{\Delta v}^2\) has conditional expectation zero. It is therefore zero
conditionally almost surely. Both \(\omega^2\gsnorm{\Delta v}^2\) and
\(\omega\gsip{\Delta v} I\) vanish, so \eqref{eq:gs-conditional-quadratic}
reduces to \(d(r)\), independently of \(\omega\).
\end{proof}

If \(A=\gsnorm{{\Delta v}}^2>0\) and \(c=\gsip{{\Delta v}}{I}/A\), the
unclipped optimum is \(\gsE[Ac\mid R]/\gsE[A\mid R]\).
Thus calibration for ordinary field MSE is innovation-energy weighted;
it is generally not the unweighted conditional average of \(c\).
An unweighted average instead corresponds to a per-sample loss normalized
by \(A\), under the integrability conditions needed for that loss.

\subsection{Why Smoothness Cannot Determine the Retention Coefficient}
\label{app:gs-nonidentifiability}

\begin{gsproposition}[Smoothness alone does not identify retention]
\label{prop:gs-nonidentifiability}
Fix three distinct anchor times and their vector-valued outputs, with
\({\Delta v_q}\ne0\), and fix a distinct target time. For every
\(c_0\in\gsR\), there exists a \(C^\infty\) model-output trace that matches
all three anchor outputs and whose target coefficient in
\eqref{eq:gs-oracle-coefficient} equals \(c_0\). The trace can be
realized by a smooth time-dependent flow field that is independent of the state.
\end{gsproposition}

\begin{proof}[Proof of Proposition~\ref{prop:gs-nonidentifiability}]
Let the three anchor times be \(t_0,t_1,t_2\), let the target time be
\(t_3\), and let the corresponding fixed anchor outputs be
\(v_0,v_1,v_2\). All four times are distinct. Define
\[
  {\Delta v}=v_2-v_1\ne0
\]
and fix an arbitrary desired coefficient \(c_0\in\gsR\). Choose the target
output
\[
  v_3=v_2+c_0{\Delta v}.
\]
For \(i\in\{0,1,2,3\}\), define the scalar Lagrange basis polynomial
\[
  L_i(t)=\prod_{\substack{0\le j\le3\\j\ne i}}
  \frac{t-t_j}{t_i-t_j}.
\]
Every denominator is nonzero because the times are distinct. Moreover,
\(L_i(t_j)=1\) when \(i=j\) and \(L_i(t_j)=0\) when \(i\ne j\). The
vector-valued polynomial
\begin{equation}
  g_{c_0}(t)=\sum_{i=0}^{3}v_iL_i(t)
  \label{eq:gs-lagrange-trace}
\end{equation}
therefore satisfies
\[
  g_{c_0}(t_i)=v_i,
  \qquad i=0,1,2,3.
\]
It is a polynomial of degree at most three in each coordinate, hence it is
\(C^\infty\). At the target,
\[
  I=g_{c_0}(t_3)-g_{c_0}(t_2)=v_3-v_2=c_0{\Delta v}.
\]
Substitution into \eqref{eq:gs-oracle-coefficient} gives
\[
  c=\frac{\gsip{{\Delta v}}{c_0{\Delta v}}}{\gsnorm{{\Delta v}}^2}=c_0.
\]
The three anchor outputs and all statistics computed from them, including
\(\sin^2\theta\), are identical for every \(c_0\); the times and hence \(\rho\) are
also identical. Finally, define a state-independent time-dependent field
\[
  V_{c_0}(x,t)=g_{c_0}(t).
\]
This field is \(C^\infty\) in \((x,t)\), and its output along every state
trace is exactly \(g_{c_0}(t)\). Thus the construction lies inside the class
of smooth flow fields and completes the proof.
\end{proof}

The preceding construction can be strengthened: even revealing any fixed
finite number of derivatives at every anchor does not identify the target
coefficient.

\begin{gsproposition}[Finite anchor jets do not identify retention]
\label{prop:gs-finite-jets}
Fix an integer \(m\ge0\), distinct anchor times \(t_0,t_1,t_2\), a
distinct target time \(t_3\), and a \(C^\infty\) trace \(g_0\) with
\({\Delta v}=g_0(t_2)-g_0(t_1)\ne0\). There is a one-parameter family of
\(C^\infty\) traces \(g_\alpha\) such that
\begin{equation}
  g_\alpha^{(k)}(t_j)=g_0^{(k)}(t_j),
  \qquad j\in\{0,1,2\},\quad 0\le k\le m,
  \label{eq:gs-matching-jets}
\end{equation}
while the oracle coefficient at \(t_3\) ranges over all of \(\gsR\).
Each trace is realizable by a state-independent smooth flow field.
\end{gsproposition}

\begin{proof}
Define the scalar polynomial
\begin{equation}
  \phi(t)=\prod_{j=0}^{2}(t-t_j)^{m+1}
  \label{eq:gs-jet-polynomial}
\end{equation}
and the vector-valued trace
\begin{equation}
  g_\alpha(t)=g_0(t)+\alpha\phi(t){\Delta v},
  \qquad \alpha\in\gsR.
  \label{eq:gs-jet-family}
\end{equation}
Fix an anchor \(t_j\). Factoring the corresponding term gives
\[
  \phi(t)=(t-t_j)^{m+1}\psi_j(t),
  \qquad
  \psi_j(t)=\prod_{\substack{0\le i\le2\\i\ne j}}
              (t-t_i)^{m+1}.
\]
For \(0\le k\le m\), Leibniz' rule gives
\[
  \phi^{(k)}(t)
  =\sum_{\ell=0}^{k}\binom{k}{\ell}
    \frac{d^\ell}{dt^\ell}(t-t_j)^{m+1}
    \psi_j^{(k-\ell)}(t).
\]
Every \(\ell\) in this sum satisfies \(\ell\le k\le m\), and hence
\[
  \frac{d^\ell}{dt^\ell}(t-t_j)^{m+1}
  =\frac{(m+1)!}{(m+1-\ell)!}
    (t-t_j)^{m+1-\ell}
\]
still contains a positive power of \(t-t_j\). Therefore
\(\phi^{(k)}(t_j)=0\). Differentiating
\eqref{eq:gs-jet-family} now proves
\eqref{eq:gs-matching-jets} term by term.

In particular, \(g_\alpha(t_j)=g_0(t_j)\) at all three anchors, so the
latest innovation remains \({\Delta v}\). Since \(\phi(t_2)=0\), the target
innovation is
\begin{align*}
  I_\alpha
  &=g_\alpha(t_3)-g_\alpha(t_2)\\
  &=g_0(t_3)-g_0(t_2)+\alpha\phi(t_3){\Delta v}.
\end{align*}
Let \(c_0=\gsip{{\Delta v}}{g_0(t_3)-g_0(t_2)}/\gsnorm {\Delta v}^2\). Then
\begin{align*}
  c_\alpha
  &=\frac{\gsip {\Delta v}{I_\alpha}}{\gsnorm {\Delta v}^2}\\
  &=c_0+\alpha\phi(t_3).
\end{align*}
The target differs from every anchor, so every factor in
\(\phi(t_3)\) is nonzero and \(\phi(t_3)\ne0\). As \(\alpha\) ranges over
\(\gsR\), so does \(c_\alpha\). Finally,
\(V_\alpha(x,t)=g_\alpha(t)\) is a state-independent \(C^\infty\) field
realizing the trace.
\end{proof}

The interpolation proofs are global and exact. A local expansion shows more
specifically why projective turning and the oracle coefficient contain
different differential information.

\begin{gsproposition}[Tangential-normal separation]
\label{prop:gs-tangent-normal}
Let \(g\in C^3\), fix a time \(t\) with \(v=g'(t)\ne0\), and put
\(a=g''(t)\). Use equally spaced anchor times \(t-2h,t-h,t\), and a target
time \(t+\lambda h\) with fixed \(\lambda>0\). Let \(c_h\) be the oracle
coefficient and \(\sin^2\theta_h\) the projective statistic. Then
\begin{align}
  c_h
  &=\lambda+
    \frac{\lambda(1+\lambda)}{2}h
    \frac{\gsip{v}{a}}{\gsnorm{v}^2}
    +O(h^2),
  \label{eq:gs-oracle-local-expansion}\\
  \sin^2\theta_h
  &=h^2\frac{\gsnorm{P_v^\perp a}^2}{\gsnorm{v}^2}
    +O(h^3).
  \label{eq:gs-chi-local-expansion}
\end{align}
Thus the first correction to \(c_h\) depends on tangential log-speed
\(\gsip{v}{a}/\gsnorm v^2\), whereas \(\sin^2\theta_h\) measures normal turning.
\end{gsproposition}

\begin{proof}
Taylor expansion about \(t\) gives
\begin{align*}
  g(t-h)
  &=g(t)-hv+\frac{h^2}{2}a+O(h^3),\\
  g(t-2h)
  &=g(t)-2hv+2h^2a+O(h^3),\\
  g(t+\lambda h)
  &=g(t)+\lambda hv+\frac{\lambda^2h^2}{2}a+O(h^3).
\end{align*}
Therefore the latest innovation, previous innovation, and target innovation
are
\begin{align}
  {\Delta v_h}
  &=g(t)-g(t-h)
   =hv-\frac{h^2}{2}a+O(h^3),
  \label{eq:gs-latest-taylor}\\
  {\Delta v_h^{-}}
  &=g(t-h)-g(t-2h)
   =hv-\frac{3h^2}{2}a+O(h^3),
  \label{eq:gs-previous-taylor}\\
  I_h
  &=g(t+\lambda h)-g(t)
   =\lambda hv+\frac{\lambda^2h^2}{2}a+O(h^3).
  \label{eq:gs-target-taylor}
\end{align}
Using bilinearity of the inner product in
\eqref{eq:gs-latest-taylor} and \eqref{eq:gs-target-taylor},
\begin{align*}
  \gsip{{\Delta v_h}}{I_h}
  &=\lambda h^2\gsnorm v^2
    +\frac{\lambda(\lambda-1)}{2}h^3\gsip va
    +O(h^4),\\
  \gsnorm{{\Delta v_h}}^2
  &=h^2\gsnorm v^2-h^3\gsip va+O(h^4).
\end{align*}
After cancelling \(h^2\), use
\((A+hB)^{-1}=A^{-1}-hB/A^2+O(h^2)\) with
\(A=\gsnorm v^2\) to obtain
\begin{align*}
  c_h
  &=\frac{\lambda\gsnorm v^2
      +\frac{\lambda(\lambda-1)}2h\gsip va+O(h^2)}
     {\gsnorm v^2-h\gsip va+O(h^2)}\\
  &=\lambda+
    \left[\frac{\lambda(\lambda-1)}2+\lambda\right]
    h\frac{\gsip va}{\gsnorm v^2}+O(h^2)\\
  &=\lambda+
    \frac{\lambda(1+\lambda)}2h
    \frac{\gsip va}{\gsnorm v^2}+O(h^2),
\end{align*}
which proves \eqref{eq:gs-oracle-local-expansion}.

It remains to expand the angle. Let \(u=v/\gsnorm v\). For any fixed
vector \(b\), direct expansion of the norm shows
\begin{equation}
  \frac{v+hb+O(h^2)}{\gsnorm{v+hb+O(h^2)}}
  =u+\frac{h}{\gsnorm v}P_v^\perp b+O(h^2).
  \label{eq:gs-normalization-expansion}
\end{equation}
Indeed,
\[
  \gsnorm{v+hb+O(h^2)}
  =\gsnorm v+h\gsip ub+O(h^2),
\]
and substituting the reciprocal of this scalar yields
\(b-u\gsip ub=P_v^\perp b\) in the first-order term.
Divide \eqref{eq:gs-latest-taylor} and
\eqref{eq:gs-previous-taylor} by \(h\) and apply
\eqref{eq:gs-normalization-expansion}. Their unit directions obey
\begin{align*}
  u_{\Delta v}
  &=u-\frac{h}{2\gsnorm v}P_v^\perp a+O(h^2),\\
  u_{\Delta v^{-}}
  &=u-\frac{3h}{2\gsnorm v}P_v^\perp a+O(h^2).
\end{align*}
Hence
\[
  u_{\Delta v}-u_{\Delta v^{-}}
  =\frac{h}{\gsnorm v}P_v^\perp a+O(h^2).
\]
Because \(P_{u_{\Delta v^{-}}}^\perp u_{\Delta v^{-}}=0\) and \(u_{\Delta v^{-}}=u+O(h)\),
\begin{align*}
  \sin\theta_h
  &=\gsnorm{P_{u_{\Delta v^{-}}}^\perp u_{\Delta v}}\\
  &=\gsnorm{P_{u_{\Delta v^{-}}}^\perp(u_{\Delta v}-u_{\Delta v^{-}})}\\
  &=\frac{h}{\gsnorm v}\gsnorm{P_v^\perp a}+O(h^2).
\end{align*}
Squaring this identity proves \eqref{eq:gs-chi-local-expansion}.
\end{proof}

Even when tangential speed is constant, the observable closure is not the
oracle coefficient of every smooth curved trace. The following exact example
keeps this distinction visible.

\begin{gsproposition}[Constant-speed circular sanity check]
\label{prop:gs-circle}
Let \(g(t)=(\cos(\Omega t),\sin(\Omega t))\), use anchors at
\(-2h,-h,0\), and predict the target at \(h\). Assume
\(\sin(\Omega h/2)\ne0\). Then
\begin{equation}
  \sin^2\theta_h=\sin^2(\Omega h),
  \qquad
  c_h=\cos(\Omega h).
  \label{eq:gs-circle-oracle}
\end{equation}
At unit normalized horizon, GeoShrink instead uses
\begin{equation}
  \omega_h=\frac{1}{1+\sin^2(\Omega h)}.
  \label{eq:gs-circle-geoshrink}
\end{equation}
Consequently,
\begin{equation}
\begin{aligned}
  c_h&=1-\frac{\Omega^2h^2}{2}+O(h^4),\\
  \omega_h&=1-\Omega^2h^2+O(h^4).
\end{aligned}
\label{eq:gs-circle-expansion}
\end{equation}
so the two coefficients already differ at order \(h^2\).
\end{gsproposition}

\begin{proof}
For arbitrary \(u\), the angle-addition identities give
\begin{align*}
  g(u+h)-g(u)
  &=\bigl(\cos(\Omega(u+h))-\cos(\Omega u),
          \sin(\Omega(u+h))-\sin(\Omega u)\bigr)\\
  &=2\sin\left(\frac{\Omega h}{2}\right)
    \left(-\sin\left(\Omega u+\frac{\Omega h}{2}\right),
           \cos\left(\Omega u+\frac{\Omega h}{2}\right)\right).
\end{align*}
Thus every length-\(h\) chord has magnitude
\(2|\sin(\Omega h/2)|\), and its oriented unit direction is the unit
tangent at the chord midpoint, up to a common sign. The previous innovation,
latest innovation, and target innovation have midpoints
\(-3h/2,-h/2,h/2\), respectively. The angle between the previous and latest
chord lines is therefore \(\Omega h\) modulo sign, and the squared projective
sine is \(\sin^2(\Omega h)\). This proves the first identity in
\eqref{eq:gs-circle-oracle}.

The latest and target innovations have equal magnitude, and their oriented
directions differ by \(\Omega h\). Hence
\[
  c_h
  =\frac{\gsip{{\Delta v_h}}{I_h}}{\gsnorm{{\Delta v_h}}^2}
  =\cos(\Omega h),
\]
which proves the second identity. At unit normalized horizon,
\(R=\sin^2\theta_h\), so \eqref{eq:gs-retention} gives
\eqref{eq:gs-circle-geoshrink}. Finally, with \(x=\Omega h\),
\[
  \cos x=1-\frac{x^2}{2}+O(x^4),
  \qquad
  \sin^2x=x^2+O(x^4),
\]
and
\[
  \frac{1}{1+\sin^2x}
  =1-\sin^2x+O(\sin^4x)
  =1-x^2+O(x^4).
\]
Substituting \(x=\Omega h\) proves
\eqref{eq:gs-circle-expansion}.
\end{proof}

\subsection{Projective Residual and Projector Identities}
\label{app:gs-projectors}

For nonzero vectors \(a,b\),
\[
 P_b=\frac{bb^\top}{\gsnorm b^2},\qquad
 \gsnorm{P_ba}^2=\frac{(a^\top b)^2}{\gsnorm b^2}.
\]
The orthogonality of \(P_ba\) and \(P_b^\perp a\) implies
\begin{equation}
 \frac{\gsnorm{P_b^\perp a}^2}{\gsnorm a^2}
 =1-\frac{(a^\top b)^2}{\gsnorm a^2\gsnorm b^2}.
 \label{eq:gs-projector-identity}
\end{equation}
Let \(u=a/\gsnorm a\) and \(v=b/\gsnorm b\). Since \(u\) and \(v\)
are unit vectors,
\begin{align*}
 \|uu^\top-vv^\top\|_F^2
 &=\operatorname{tr}(uu^\top uu^\top)
  +\operatorname{tr}(vv^\top vv^\top)
  -2\operatorname{tr}(uu^\top vv^\top)\\
 &=2-2(u^\top v)^2.
\end{align*}
Taking \(a={\Delta v_q}\), \(b={\Delta v_q^{-}}\) proves
\eqref{eq:gs-projector-identity}. Nonzero scalar multiplication
does not change a line projector. Orthogonal coordinate transformations
conjugate the projectors and preserve their Frobenius distance.

\subsection{Why the Update Midpoint Appears}
\label{app:gs-midpoint}

The midpoint horizon used in \eqref{eq:gs-rho} satisfies
\begin{equation}
  {
  \int_0^1r(u)\,du=\rho_{q,n},
  \qquad
  \int_0^1r(u)^2\,du
  =\rho_{q,n}^2+\frac1{12}
    \left(\frac{\sigma_n-\sigma_{n-1}}
    {\sigma_n^{(1)}-\sigma_n^{(2)}}\right)^2.}
  \label{eq:gs-midpoint-moments}
\end{equation}

Let
\[
  H_q=\sigma_n^{(1)}-\sigma_n^{(2)}\ne0,
  \qquad
  r_0=\frac{\sigma_n-\sigma_n^{(1)}}{H_q},
  \qquad
  r_1=\frac{\sigma_{n-1}-\sigma_n^{(1)}}{H_q}.
\]
The normalized displacement at the affine interval coordinate
\(u\in[0,1]\) is
\begin{align*}
  r(u)
  &=\frac{(1-u)\sigma_n+u\sigma_{n-1}-\sigma_n^{(1)}}{H_q}\\
  &=(1-u)r_0+ur_1.
\end{align*}
Define its midpoint and endpoint difference by
\[
  \bar r=\frac{r_0+r_1}{2},
  \qquad
  \Delta r=r_1-r_0.
\]
Then
\begin{equation}
\begin{aligned}
  r(u)
  &=(1-u)\left(\bar r-\frac{\Delta r}{2}\right)
    +u\left(\bar r+\frac{\Delta r}{2}\right)\\
  &=\bar r+\left(u-\frac12\right)\Delta r.
\end{aligned}
\label{eq:gs-centered-horizon}
\end{equation}
By \eqref{eq:gs-rho}, \(\bar r=\rho_{q,n}\). Also,
\begin{align*}
  \int_0^1\left(u-\frac12\right)\,du
  &=\left[\frac{u^2}{2}-\frac{u}{2}\right]_{0}^{1}=0,\\
  \int_0^1\left(u-\frac12\right)^2\,du
  &=\int_0^1\left(u^2-u+\frac14\right)\,du\\
  &=\left[\frac{u^3}{3}-\frac{u^2}{2}+\frac{u}{4}\right]_{0}^{1}\\
  &=\frac13-\frac12+\frac14=\frac1{12}.
\end{align*}
Integrating \eqref{eq:gs-centered-horizon} therefore gives
\begin{align*}
  \int_0^1r(u)\,du
  &=\bar r+\Delta r
    \int_0^1\left(u-\frac12\right)\,du\\
  &=\bar r=\rho_{q,n}.
\end{align*}
Squaring the same identity first gives
\[
  r(u)^2
  =\bar r^2+2\bar r\Delta r\left(u-\frac12\right)
   +(\Delta r)^2\left(u-\frac12\right)^2.
\]
After integration, the cross term vanishes and the last integral equals
\(1/12\), so
\begin{align*}
  \int_0^1r(u)^2\,du
  &=\bar r^2+\frac{(\Delta r)^2}{12}\\
  &=\rho_{q,n}^2+\frac1{12}
    \left(\frac{\sigma_{n-1}-\sigma_n}{H_q}\right)^2\\
  &=\rho_{q,n}^2+\frac1{12}
    \left(\frac{\sigma_n-\sigma_{n-1}}
    {\sigma_n^{(1)}-\sigma_n^{(2)}}\right)^2.
\end{align*}
This proves both identities in \eqref{eq:gs-midpoint-moments}.
The implemented rule uses the square of the mean horizon, whereas the exact
mean squared horizon contains the nonnegative within-update variance. The
difference vanishes quadratically as the solver interval becomes small
relative to the latest anchor span.

\subsection{Continuum Consistency under \texorpdfstring{\(C^2\)}{C2} Regularity}
\label{app:gs-continuum}

\begin{gstheorem}[Continuum limit of the chordal tangent and exposure]
\label{thm:gs-continuum}
Let \(g\) be \(C^2\) near \(t_\star\), with \(g'(t_\star)\ne0\).
Let \(t_0<t_1<t_2\) converge to \(t_\star\), write
\(h_1=t_1-t_0\), \(h_2=t_2-t_1\), and suppose
\(h_1/h_2\to\eta\in(0,\infty)\). Let \(u_1,u_2\) be the two unit
secant directions, let \(m_j=(t_{j-1}+t_j)/2\), and orient each sign so that
\(\gsip{u_j}{U(m_j)}>0\), where
\(U(t)=g'(t)/\gsnorm{g'(t)}\). Let \(\theta_h\) be their projective angle.
Then
\begin{equation}
  {
  \begin{aligned}
  -\frac{P_{u_2}^\perp u_1}{h_2}
  &\longrightarrow
  \frac{1+\eta}{2}\,U'(t_\star),\\
  \frac{\sin\theta_h}{h_2}
  &\longrightarrow
  \frac{1+\eta}{2}\,\omega(t_\star).
  \end{aligned}}
  \label{eq:gs-continuum}
\end{equation}
where
\begin{equation}
  \omega(t)=\gsnorm{U'(t)}
  =\frac{\gsnorm{P_{g'(t)}^\perp g''(t)}}{\gsnorm{g'(t)}}.
  \label{eq:gs-angular-speed}
\end{equation}
If \(\delta_h\to\delta\) and \(\rho_h=\delta_h/h_2\), then
\begin{equation}
  {
  \begin{aligned}
  -\rho_hP_{u_2}^\perp u_1
  &\longrightarrow
  \delta\frac{1+\eta}{2}U'(t_\star),\\
  \rho_h^2\sin^2\theta_h
  &\longrightarrow
  \delta^2\left(\frac{1+\eta}{2}\right)^2
  \omega(t_\star)^2.
  \end{aligned}}
  \label{eq:gs-exposure-limit}
\end{equation}
\end{gstheorem}

\begin{proof}[Proof of Theorem~\ref{thm:gs-continuum}]
Define the normalized tangent
\[
  U(t)=\frac{g'(t)}{\gsnorm{g'(t)}}.
\]
Because \(g'\) is continuous and \(g'(t_\star)\ne0\), there are a
neighborhood \(\mathcal N\) of \(t_\star\) and a number \(b>0\) such that
\(\gsnorm{g'(t)}\ge b\) for every \(t\in\mathcal N\). Hence normalization
is well defined on \(\mathcal N\). Put \(s(t)=\gsnorm{g'(t)}\). Direct
differentiation gives
\[
  s'(t)=\frac{\gsip{g'(t)}{g''(t)}}{\gsnorm{g'(t)}}.
\]
Applying the quotient rule to \(U(t)=g'(t)/s(t)\) yields
\begin{align*}
  U'(t)
  &=\frac{g''(t)}{s(t)}
    -\frac{g'(t)s'(t)}{s(t)^2}\\
  &=\frac{g''(t)}{\gsnorm{g'(t)}}
    -\frac{g'(t)\gsip{g'(t)}{g''(t)}}
          {\gsnorm{g'(t)}^3}\\
  &=\frac{1}{\gsnorm{g'(t)}}
    \left(I-\frac{g'(t)g'(t)^\top}{\gsnorm{g'(t)}^2}\right)g''(t)\\
  &=\frac{P_{g'(t)}^\perp g''(t)}{\gsnorm{g'(t)}}.
\end{align*}
Therefore \(U(t)^\top U'(t)=0\) and
\(\gsnorm{U'(t)}=\omega(t)\).

For \(j\in\{1,2\}\), define the interval midpoint and secant velocity
\[
  m_j=\frac{t_{j-1}+t_j}{2},\qquad
  s_j=\frac{g(t_j)-g(t_{j-1})}{h_j}.
\]
The fundamental theorem of calculus gives
\begin{align*}
  s_j
  &=\frac1{h_j}\int_{t_{j-1}}^{t_j}g'(t)\,dt\\
  &=\frac1{h_j}\int_{-h_j/2}^{h_j/2}g'(m_j+u)\,du.
\end{align*}
Since \(g''\) is continuous, uniformly for \(|u|\le h_j/2\),
\[
  g'(m_j+u)
  =g'(m_j)+u g''(m_j)+u r_j(u),
  \qquad
  \sup_{|u|\le h_j/2}\gsnorm{r_j(u)}=o(1).
\]
Substitution into the integral gives
\begin{align*}
  s_j-g'(m_j)
  &=\frac{g''(m_j)}{h_j}
      \int_{-h_j/2}^{h_j/2}u\,du
    +\frac1{h_j}\int_{-h_j/2}^{h_j/2}u r_j(u)\,du\\
  &=\frac1{h_j}\int_{-h_j/2}^{h_j/2}u r_j(u)\,du,
\end{align*}
because the integral of the odd function \(u\) over the symmetric interval
is zero. Consequently,
\begin{equation}
\begin{aligned}
  \gsnorm{s_j-g'(m_j)}
  &\le\frac1{h_j}
      \sup_{|u|\le h_j/2}\gsnorm{r_j(u)}
      \int_{-h_j/2}^{h_j/2}|u|\,du\\
  &=\frac{h_j}{4}
      \sup_{|u|\le h_j/2}\gsnorm{r_j(u)}
   =o(h_j).
\end{aligned}
  \label{eq:gs-c2-secant}
\end{equation}

The map \(N(x)=x/\gsnorm{x}\) is continuously differentiable wherever
\(x\ne0\), with derivative
\[
  DN(x)=\frac1{\gsnorm x}
  \left(I-\frac{xx^\top}{\gsnorm x^2}\right).
\]
\Eqref{eq:gs-c2-secant} and continuity give
\(s_j\to g'(t_\star)\). Hence, for all sufficiently small spans,
\(\gsnorm{s_j-g'(m_j)}\le b/2\). Every point on the segment joining them has
the form \(x_\lambda=g'(m_j)+\lambda[s_j-g'(m_j)]\), with
\(0\le\lambda\le1\), and satisfies
\[
  \gsnorm{x_\lambda}
  \ge\gsnorm{g'(m_j)}-\lambda\gsnorm{s_j-g'(m_j)}
  \ge b/2.
\]
Thus the displayed derivative satisfies
\(\|DN(x_\lambda)\|_{\rm op}\le2/b\) along the entire segment. The
mean-value inequality therefore yields
\begin{align*}
  \left\|
  \frac{s_j}{\gsnorm{s_j}}-
  \frac{g'(m_j)}{\gsnorm{g'(m_j)}}
  \right\|
  &\le\frac2b\gsnorm{s_j-g'(m_j)}\\
  &=o(h_j).
\end{align*}
Therefore
\[
  u_j:=\frac{s_j}{\gsnorm{s_j}}
  =U(m_j)+o(h_j).
\]
In particular, both \(u_1\) and \(u_2\) converge to \(U(t_\star)\), so
their inner product is positive for all sufficiently small intervals. The
ordinary angle and the projective angle then have the same sine.

The midpoint separation is
\[
  d_h=m_2-m_1=\frac{h_1+h_2}{2}.
\]
The assumption \(h_1/h_2\to\eta\in(0,\infty)\) implies
\(d_h\asymp h_1\asymp h_2\). Hence
\(o(h_1)+o(h_2)=o(d_h)\). By the fundamental theorem of calculus applied
to \(U\),
\begin{align*}
  u_2-u_1
  &=U(m_2)-U(m_1)+o(d_h)\\
  &=\int_{m_1}^{m_2}U'(t)\,dt+o(d_h)\\
  &=d_h U'(t_\star)+o(d_h),
\end{align*}
where the last equality follows from continuity of \(U'\) and
\(m_1,m_2\to t_\star\).

The vector used by the chordal retraction is based at the latest
secant \(u_2\). Since \(P_{u_2}^\perp u_2=0\),
\begin{align*}
  P_{u_2}^\perp u_1
  &=P_{u_2}^\perp(u_1-u_2)\\
  &=-d_hP_{u_2}^\perp U'(t_\star)+o(d_h).
\end{align*}
Because \(u_2\to U(t_\star)\) and
\(U'(t_\star)\perp U(t_\star)\), division by \(d_h\) gives the vector
limit
\[
  -\frac{P_{u_2}^\perp u_1}{d_h}
  \longrightarrow U'(t_\star).
\]
The two projected residuals have the same norm:
\begin{align*}
  \gsnorm{P_{u_2}^\perp u_1}^2
  &=1-\gsip{u_1}{u_2}^2\\
  &=\gsnorm{P_{u_1}^\perp u_2}^2
  =\sin^2\theta_h.
\end{align*}
Taking norms in the vector limit therefore gives
\[
  \frac{\sin\theta_h}{d_h}
  \longrightarrow\gsnorm{U'(t_\star)}=\omega(t_\star).
\]
Finally,
\[
  \frac{d_h}{h_2}
  =\frac{h_1+h_2}{2h_2}
  =\frac{1+h_1/h_2}{2}
  \longrightarrow\frac{1+\eta}{2}.
\]
Multiplying both the vector limit and its norm limit by
\(d_h/h_2\) proves the two limits in
\eqref{eq:gs-continuum}.

For the transported tangent, use \(\rho_h=\delta_h/h_2\) to write
\begin{align*}
  -\rho_hP_{u_2}^\perp u_1
  &=\delta_h\left(-\frac{P_{u_2}^\perp u_1}{h_2}\right)\\
  &\longrightarrow
    \delta\frac{1+\eta}{2}U'(t_\star).
\end{align*}
For its squared magnitude,
\begin{align*}
  \rho_h^2\sin^2\theta_h
  &=\delta_h^2
    \left(\frac{\sin\theta_h}{h_2}\right)^2\\
  &\longrightarrow
    \delta^2\left(\frac{1+\eta}{2}\right)^2
    \omega(t_\star)^2.
\end{align*}
These are the two limits in
\eqref{eq:gs-exposure-limit}. The square in the GeoShrink exposure
is used because the transported directional deviation enters a squared-error
objective; the theorem does not claim that a remote-horizon asymptotic
uniquely determines this exponent.
\end{proof}

\begin{gscorollary}[Shape-regular bounds without a limiting gap ratio]
\label{cor:gs-shape-regular}
Keep the definitions, \(C^2\) regularity, nonvanishing derivative, and
convergence \(t_0,t_1,t_2\to t_\star\) from
Theorem~\ref{thm:gs-continuum}, but do not assume that \(h_1/h_2\)
converges. Instead, suppose that for some \(\Gamma\ge1\),
\[
  \Gamma^{-1}\le\frac{h_1}{h_2}\le\Gamma
\]
for all sufficiently small intervals. Then
\begin{equation}
  \frac{\sin\theta_h}{h_2}
  =\frac{1+h_1/h_2}{2}\,\omega(t_\star)+o(1).
  \label{eq:gs-shape-regular-expansion}
\end{equation}
If additionally \(\delta_h\to\delta\), then
\begin{align}
  \delta^2\left(\frac{1+\Gamma^{-1}}{2}\right)^2
  \omega(t_\star)^2
  &\le\liminf_h\rho_h^2\sin^2\theta_h,
  \label{eq:gs-shape-liminf}\\
  \limsup_h\rho_h^2\sin^2\theta_h
  &\le
  \delta^2\left(\frac{1+\Gamma}{2}\right)^2
  \omega(t_\star)^2.
  \label{eq:gs-shape-limsup}
\end{align}
\end{gscorollary}

\begin{proof}
The proof above established
\(\sin\theta_h/d_h\to\omega(t_\star)\) before using convergence of
\(h_1/h_2\). The shape-regular bounds imply \(d_h\asymp h_2\), so the same
remainder remains \(o(1)\) after multiplication by
\(d_h/h_2=(1+h_1/h_2)/2\). This proves
\eqref{eq:gs-shape-regular-expansion}. Squaring, multiplying by
\(\delta_h^2\), and using
\[
  \frac{1+\Gamma^{-1}}2
  \le\frac{1+h_1/h_2}{2}
  \le\frac{1+\Gamma}2
\]
gives \eqref{eq:gs-shape-liminf} and
\eqref{eq:gs-shape-limsup}.
\end{proof}

\begin{gscorollary}[A quantitative local refinement]
\label{cor:gs-quantitative}
Under the conditions of Theorem~\ref{thm:gs-continuum}, write
\(h=\max\{h_1,h_2\}\). Suppose additionally that \(g\in C^3\) near
\(t_\star\) and
\(\max_{0\le j\le2}|t_j-t_\star|=O(h)\). Then
\begin{equation}
 \sin\theta_h=d_h\omega(t_\star)+O(h^2).
 \label{eq:gs-sine-rate}
\end{equation}
For bounded \(\delta_h\) and \(\rho_h=\delta_h/h_2\),
\begin{equation}
 \rho_h^2\sin^2\theta_h
 =\delta_h^2\left(\frac{1+h_1/h_2}{2}\right)^2
   \omega(t_\star)^2+O(\delta_h^2 h).
 \label{eq:gs-exposure-rate}
\end{equation}
\end{gscorollary}

\begin{proof}
A midpoint Taylor expansion now gives
\(s_j=g'(m_j)+O(h_j^2)\). The normalized secants therefore obey
\(u_j=U(m_j)+O(h_j^2)\). Since \(U\in C^2\) and the observation
midpoints are \(O(h)\) from \(t_\star\),
\[
 u_2-u_1=d_h U'(t_\star)+O(h^2),\qquad
 u_1=U(t_\star)+O(h).
\]
Because \(P_{u_2}^\perp u_2=0\),
\begin{align*}
 P_{u_2}^\perp u_1
 &=P_{u_2}^\perp(u_1-u_2)\\
 &=-d_hP_{u_2}^\perp U'(t_\star)+O(h^2).
\end{align*}
Moreover, \(u_2=U(t_\star)+O(h)\) and
\(U'(t_\star)\perp U(t_\star)\). Since
\(P_{u_2}^\perp=I-u_2u_2^\top\),
\begin{align*}
 P_{u_2}^\perp U'(t_\star)-U'(t_\star)
 &=-u_2u_2^\top U'(t_\star)\\
 &=-u_2\gsip{u_2-U(t_\star)}{U'(t_\star)}\\
 &=O(h).
\end{align*}
Consequently,
\[
 P_{u_2}^\perp u_1
 =-d_hU'(t_\star)+O(h^2),
\]
where \(d_h=O(h)\). Taking norms and using \(d_h>0\) gives
\begin{align*}
 \sin\theta_h
 &=\gsnorm{P_{u_2}^\perp u_1}\\
 &=d_h\gsnorm{U'(t_\star)}+O(h^2)\\
 &=d_h\omega(t_\star)+O(h^2),
\end{align*}
which proves \eqref{eq:gs-sine-rate}. Squaring gives
\begin{align*}
 \sin^2\theta_h
 &=\left[d_h\omega(t_\star)+O(h^2)\right]^2\\
 &=d_h^2\omega(t_\star)^2+O(d_hh^2)+O(h^4)\\
 &=d_h^2\omega(t_\star)^2+O(h^3).
\end{align*}
Because \(h_1/h_2\to\eta\in(0,\infty)\), we have \(h_2\asymp h\).
Using \(\rho_h=\delta_h/h_2\),
\begin{align*}
 \rho_h^2\sin^2\theta_h
 &=\frac{\delta_h^2}{h_2^2}
   \left[d_h^2\omega(t_\star)^2+O(h^3)\right]\\
 &=\delta_h^2
   \left(\frac{d_h}{h_2}\right)^2
   \omega(t_\star)^2+O(\delta_h^2h)\\
 &=\delta_h^2
   \left(\frac{1+h_1/h_2}{2}\right)^2
   \omega(t_\star)^2+O(\delta_h^2h),
\end{align*}
which proves \eqref{eq:gs-exposure-rate}.
\end{proof}

The localization condition in this corollary is needed to state the
\(O(h^2)\) rate at a fixed \(t_\star\). Convergence of all observation
times to \(t_\star\) alone suffices for the main theorem's limit, but does
not specify their distance from \(t_\star\) relative to their spacing.
Neither result predicts a finite remote tangent from local information.
Applying the interpretation to a sampler additionally requires a smooth
model-output trace; the lemmas alone do not establish smoothness of an
accelerated trajectory.

\subsection{Complete Derivation of Chordal Projective Retention}
\label{app:gs-variational}

\begin{gstheorem}[Chordal transport and cache projection]
\label{thm:gs-projective-retention}
Under the chordal transport construction of Proposition~\ref{prop:gs-retention},
let \({\Delta v}={\Delta v_q}\), \({\Delta v^{-}}={\Delta v_q^{-}}\),
\(L_q=\gsspan({\Delta v})\), and
\(\widetilde L_{q,n}=\gsspan(\widetilde u_{q,n})\). The two projection
problems
\begin{align}
  y_{q,n}
  &=\arg\min_{y\in\widetilde L_{q,n}}\gsnorm{y-{\Delta v}}^2,
  \label{eq:gs-first-line-projection}\\
  \widehat {\Delta v}_{q,n}
  &=\arg\min_{x\in L_q}\gsnorm{x-y_{q,n}}^2
  \label{eq:gs-second-line-projection}
\end{align}
have unique solutions, and
\begin{equation}
  {
  \widehat {\Delta v}_{q,n}
  =P_u\widetilde P_{q,n}{\Delta v}
  =\operatorname{tr}(P_u\widetilde P_{q,n}){\Delta v}
  =\frac{1}{1+\rho_{q,n}^2\sin^2\theta_q}\,{\Delta v}.}
  \label{eq:gs-transport-projection}
\end{equation}
The same scalar obeys
\begin{equation}
  \begin{aligned}
  \omega_{q,n}
  &=\frac{\gsnorm{\widetilde P_{q,n}{\Delta v}}^2}{\gsnorm {\Delta v}^2}
   =\operatorname{tr}(P_u\widetilde P_{q,n})\\
  &=1-\frac12\left\|P_u-\widetilde P_{q,n}\right\|_F^2.
  \end{aligned}
  \label{eq:gs-line-fidelity}
\end{equation}
It is also the Hilbert-Schmidt compression coefficient,
\begin{equation}
  {
  \begin{aligned}
  \omega_{q,n}
  &=\arg\min_{\gamma\in\gsR}
    \left\|\widetilde P_{q,n}-\gamma P_u\right\|_F^2,\\
  P_u\widetilde P_{q,n}P_u&=\omega_{q,n}P_u.
  \end{aligned}}
  \label{eq:gs-operator-compression}
\end{equation}
Finally, for every \(\omega\in\gsR\),
\begin{equation}
  {
  (1+R_{q,n})
  \gsnorm{\omega {\Delta v}-\widetilde P_{q,n}{\Delta v}}^2
  =\gsnorm{{\Delta v}-\omega {\Delta v}}^2
   +\rho_{q,n}^2\gsnorm{P_{\Delta v^{-}}^\perp(\omega {\Delta v})}^2.}
  \label{eq:gs-projective-variational-identity}
\end{equation}
Therefore the unique cache-compatible coefficient is exactly
\eqref{eq:gs-retention}.
\end{gstheorem}

This section proves Theorem~\ref{thm:gs-projective-retention} for the
deterministic closure stated in the main text. The closure has four explicit
parts: the scalar correction family in \eqref{eq:gs-family}, the
observable chordal tangent between the last two innovation lines, the declared
constant-tangent ambient continuation followed by normalization, and two
Euclidean orthogonal projections. Equivalently, the last operation is the
Hilbert-Schmidt compression of the transported line projector onto the
one-dimensional cache-operator family. Once these parts are fixed, the
coefficient and the quadratic identity follow without a stochastic model, a
fitted constant, or an assumed future output. Smoothness supplies only the
local scaling result in Theorem~\ref{thm:gs-continuum}; it does not uniquely
select this closure.

\paragraph{Step 1: orient the two line representatives.}
Abbreviate
\[
  {\Delta v}={\Delta v_q},\qquad {\Delta v^{-}}={\Delta v_q^{-}},\qquad
  A=\gsnorm {\Delta v}^2>0,
\]
and define the unit vectors
\[
  u=\frac{{\Delta v}}{\sqrt A},
  \qquad
  v=\frac{{\Delta v^{-}}}{\gsnorm {\Delta v^{-}}}.
\]
A projective line is unchanged when its unit representative changes sign.
Choose \(\bar v\in\{v,-v\}\) so that
\[
  \alpha=\gsip u{\bar v}=|\gsip uv|\in[0,1].
\]
When \(\gsip uv=0\), either sign may be used. Define the component of the
previous direction tangent to the unit sphere at \(u\) by
\begin{equation}
  \xi=P_u^\perp\bar v
  =(I-uu^\top)\bar v
  =\bar v-\alpha u.
  \label{eq:gs-app-tangent}
\end{equation}
Its orthogonality to \(u\) follows term by term:
\begin{align*}
  \gsip u\xi
  &=\gsip u{\bar v}-\alpha\gsip uu\\
  &=\alpha-\alpha=0.
\end{align*}
Its squared norm is
\begin{equation}
\begin{aligned}
  \gsnorm\xi^2
  &=\gsip{\bar v-\alpha u}{\bar v-\alpha u}\\
  &=\gsnorm{\bar v}^2
    -2\alpha\gsip u{\bar v}
    +\alpha^2\gsnorm u^2\\
  &=1-2\alpha^2+\alpha^2\\
  &=1-\alpha^2\\
  &=1-\gsip uv^2\\
  &=1-\frac{\gsip {\Delta v}{\Delta v^{-}}^2}{\gsnorm {\Delta v}^2\gsnorm {\Delta v^{-}}^2}\\
  &=\sin^2\theta_q.
\end{aligned}
\label{eq:gs-app-tangent-norm}
\end{equation}
This establishes the chordal-tangent identity used in
Proposition~\ref{prop:gs-retention}. In particular, \(\sqrt{\sin^2\theta_q}\) is
exactly the chordal tangent magnitude, not a postulated curvature variable.

\paragraph{Step 2: verify the normalization retraction.}
For a tangent vector \(\eta\in u^\perp\), the normalization map
\[
  \operatorname{Retr}_u(\eta)
  =\frac{u+\eta}{\gsnorm{u+\eta}}
\]
is a retraction on the unit sphere. Indeed,
\[
  \operatorname{Retr}_u(0)=u.
\]
For a scalar \(s\), orthogonality gives
\[
  \gsnorm{u+s\eta}^2
  =\gsnorm u^2+2s\gsip u\eta+s^2\gsnorm\eta^2
  =1+s^2\gsnorm\eta^2,
\]
so
\[
  \operatorname{Retr}_u(s\eta)
  =\frac{u+s\eta}{\sqrt{1+s^2\gsnorm\eta^2}}.
\]
Differentiating at \(s=0\) gives
\begin{align*}
  \left.\frac{d}{ds}\operatorname{Retr}_u(s\eta)\right|_{s=0}
  &=\left.
    \frac{\eta}{\sqrt{1+s^2\gsnorm\eta^2}}
    -\frac{s\gsnorm\eta^2(u+s\eta)}
           {(1+s^2\gsnorm\eta^2)^{3/2}}
    \right|_{s=0}\\
  &=\eta.
\end{align*}
Thus the map has the required first-order velocity. Taking
\(\eta=-\rho_{q,n}\xi\), and using
\eqref{eq:gs-app-tangent-norm}, gives
\begin{align*}
  \gsnorm{u-\rho_{q,n}\xi}^2
  &=\gsnorm u^2
    -2\rho_{q,n}\gsip u\xi
    +\rho_{q,n}^2\gsnorm\xi^2\\
  &=1+\rho_{q,n}^2\sin^2\theta_q.
\end{align*}
Set \(a_\rho=u-\rho_{q,n}\xi\). The preceding norm identity shows that
\(a_\rho\ne0\). For every unit vector \(y\), expansion followed by the
Cauchy-Schwarz inequality gives
\begin{align*}
  \gsnorm{y-a_\rho}^2
  &=\gsnorm y^2+\gsnorm{a_\rho}^2-2\gsip y{a_\rho}\\
  &=1+\gsnorm{a_\rho}^2-2\gsip y{a_\rho}\\
  &\ge1+\gsnorm{a_\rho}^2-2\gsnorm{a_\rho}.
\end{align*}
Equality in Cauchy-Schwarz holds with the required positive inner product
if and only if \(y=a_\rho/\gsnorm{a_\rho}\). Therefore the unique unit vector
closest to the affine first-order continuation is
\[
  \widetilde u
  =\operatorname{Retr}_u(-\rho_{q,n}\xi)
  =\arg\min_{\gsnorm y=1}\gsnorm{y-a_\rho}^2
  =\frac{u-\rho_{q,n}\xi}
         {\sqrt{1+\rho_{q,n}^2\sin^2\theta_q}}.
\]
It defines the orthogonal line projector
\[
  Q=\widetilde u\widetilde u^\top,
  \qquad Q^\top=Q,
  \qquad Q^2=Q.
\]
The last identity follows explicitly from
\[
  Q^2
  =\widetilde u\widetilde u^\top
   \widetilde u\widetilde u^\top
  =\widetilde u(\widetilde u^\top\widetilde u)
   \widetilde u^\top
  =\widetilde u\widetilde u^\top=Q.
\]
This establishes the normalization retraction used in
Proposition~\ref{prop:gs-retention}.

\paragraph{Step 3: solve the projection onto the extrapolated line.}
Let \(\widetilde L=\gsspan(\widetilde u)\). Every vector
\(y\in\widetilde L\) satisfies \(Qy=y\). Moreover,
\[
  \widetilde u^\top({\Delta v}-Q{\Delta v})
  =\widetilde u^\top {\Delta v}
   -\widetilde u^\top\widetilde u\widetilde u^\top {\Delta v}
  =0,
\]
so \({\Delta v}-Q{\Delta v}\) is orthogonal to \(\widetilde L\). For arbitrary
\(y\in\widetilde L\), write
\[
  {\Delta v}-y=({\Delta v}-Q{\Delta v})+(Q{\Delta v}-y).
\]
The first summand is orthogonal to \(\widetilde L\), while the second belongs
to \(\widetilde L\). The Pythagorean theorem therefore gives
\begin{equation}
  \gsnorm{{\Delta v}-y}^2
  =\gsnorm{{\Delta v}-Q{\Delta v}}^2+\gsnorm{Q{\Delta v}-y}^2
  \ge\gsnorm{{\Delta v}-Q{\Delta v}}^2.
  \label{eq:gs-first-projection-pythagoras}
\end{equation}
Equality holds exactly when \(y=Q{\Delta v}\). Hence
\[
  y_{q,n}=Q{\Delta v}
\]
is the unique solution of \eqref{eq:gs-first-line-projection}.

\paragraph{Step 4: project back to the cache-compatible line.}
Let \(L=\gsspan(u)=\gsspan({\Delta v})\) and \(P=uu^\top\). The same orthogonal
projection argument, now applied to \(Q{\Delta v}\) and the closed subspace \(L\),
gives the unique solution
\[
  \widehat {\Delta v}_{q,n}=PQ{\Delta v}
\]
of \eqref{eq:gs-second-line-projection}. It remains to evaluate this
vector. Orthogonality of \(u\) and \(\xi\) gives
\begin{align}
  \gsip u{\widetilde u}
  &=\frac{\gsip u{u-\rho_{q,n}\xi}}
          {\sqrt{1+\rho_{q,n}^2\sin^2\theta_q}}\nonumber\\
  &=\frac{\gsip uu-\rho_{q,n}\gsip u\xi}
          {\sqrt{1+\rho_{q,n}^2\sin^2\theta_q}}\nonumber\\
  &=\frac{1}{\sqrt{1+\rho_{q,n}^2\sin^2\theta_q}}.
  \label{eq:gs-line-overlap}
\end{align}
Because \({\Delta v}=\sqrt A\,u\), first applying \(Q\) gives
\begin{align*}
  Q{\Delta v}
  &=\widetilde u\widetilde u^\top(\sqrt A\,u)\\
  &=\sqrt A\,\gsip{\widetilde u}{u}\widetilde u.
\end{align*}
Applying \(P\) gives
\begin{align*}
  PQ{\Delta v}
  &=uu^\top
    \left(\sqrt A\,\gsip{\widetilde u}{u}\widetilde u\right)\\
  &=\sqrt A\,\gsip{\widetilde u}{u}
    \gsip u{\widetilde u}u\\
  &=\sqrt A\,\gsip u{\widetilde u}^{2}u\\
  &=\frac{\sqrt A}{1+\rho_{q,n}^2\sin^2\theta_q}u\\
  &=\frac{1}{1+\rho_{q,n}^2\sin^2\theta_q}{\Delta v}.
\end{align*}
The transported component has squared norm
\begin{align*}
  \gsnorm{Q{\Delta v}}^2
  &=(Q{\Delta v})^\top(Q{\Delta v})\\
  &={\Delta v}^\top Q^\top Q{\Delta v}\\
  &={\Delta v}^\top Q{\Delta v}\\
  &=A\,u^\top\widetilde u\widetilde u^\top u\\
  &=A\,\gsip u{\widetilde u}^{2}.
\end{align*}
Dividing by \(\gsnorm {\Delta v}^2=A\) shows that the retained energy fraction is
\(\gsip u{\widetilde u}^{2}\). The same scalar is the overlap of the two
rank-one projectors. Cyclicity of the trace gives
\begin{align*}
  \operatorname{tr}(PQ)
  &=\operatorname{tr}
    (uu^\top\widetilde u\widetilde u^\top)\\
  &=\operatorname{tr}
    (u^\top\widetilde u\widetilde u^\top u)\\
  &=\gsip u{\widetilde u}^{2}\\
  &=\frac{1}{1+\rho_{q,n}^2\sin^2\theta_q}.
\end{align*}
Moreover, since \(P^2=P\), \(Q^2=Q\), and both projectors have trace one,
\begin{align*}
  \left\|P-Q\right\|_F^2
  &=\operatorname{tr}[(P-Q)^\top(P-Q)]\\
  &=\operatorname{tr}(P^2)+\operatorname{tr}(Q^2)
    -2\operatorname{tr}(PQ)\\
  &=2-2\operatorname{tr}(PQ).
\end{align*}
Rearrangement proves \eqref{eq:gs-line-fidelity}. This also proves
\eqref{eq:gs-transport-projection}.

The same scalar follows from an operator-space least-squares problem. Equip
matrices with the Hilbert-Schmidt inner product
\[
  \langle A,B\rangle_F=\operatorname{tr}(A^\top B).
\]
Because \(P\) and \(Q\) are symmetric rank-one projectors,
\[
  \|P\|_F^2=\operatorname{tr}(P^2)=\operatorname{tr}(P)=1,
  \qquad
  \|Q\|_F^2=1.
\]
For any \(\gamma\in\gsR\), bilinearity of this inner product gives
\begin{align*}
  \|Q-\gamma P\|_F^2
  &=\|Q\|_F^2
    -2\gamma\langle Q,P\rangle_F
    +\gamma^2\|P\|_F^2\\
  &=1-2\gamma\operatorname{tr}(PQ)+\gamma^2\\
  &=\left[\gamma-\operatorname{tr}(PQ)\right]^2
    +1-\operatorname{tr}(PQ)^2.
\end{align*}
The final term does not depend on \(\gamma\), so the unique minimizer is
\[
  \gamma^\star=\operatorname{tr}(PQ)=\omega_{q,n}.
\]
Moreover,
\begin{align*}
  PQP
  &=uu^\top Q uu^\top\\
  &=\left(u^\top Q u\right)uu^\top\\
  &=\operatorname{tr}(PQ)P\\
  &=\omega_{q,n}P.
\end{align*}
Since \(P{\Delta v}={\Delta v}\), applying this compressed operator to the latest innovation
gives
\[
  PQP{\Delta v}=\omega_{q,n}P{\Delta v}=\omega_{q,n}{\Delta v},
\]
while \(PQP{\Delta v}=PQ{\Delta v}\). This proves every statement in
\eqref{eq:gs-operator-compression} and shows that the sequential
vector projections are the action of the Hilbert-Schmidt compression on the
cache line.

Therefore the scalar coefficient \(\omega\), the retained-energy fraction, and the
scalar projection loss are unchanged by the sign of \(\bar v\), by the sign
of \(\rho_{q,n}\), or by reversing either observed line. When
\(\gsip uv=0\), the two admissible choices of \(\bar v\) may produce
different oriented transported lines, but they have the same squared line
overlap with \(L\) and hence the same coefficient.

\paragraph{Step 5: derive the scalar least-squares problem.}
Every cache-compatible correction has the form \(x=\omega {\Delta v}\). Since
\(P Q{\Delta v}\) is the orthogonal projection of \(Q{\Delta v}\) onto \(L\), it is also the
unique minimizer of
\[
  \min_{\omega\in\gsR}\gsnorm{\omega {\Delta v}-Q{\Delta v}}^2.
\]
For completeness, this follows by direct expansion. Put
\[
  R=\rho_{q,n}^2\sin^2\theta_q,
  \qquad
  a^2=\gsip u{\widetilde u}^{2}=\frac{1}{1+R}.
\]
Because \(Q\) is symmetric and idempotent,
\begin{align*}
  \gsip {\Delta v}{Q{\Delta v}}
  &={\Delta v}^\top Q{\Delta v},\\
  \gsnorm{Q{\Delta v}}^2
  &=(Q{\Delta v})^\top(Q{\Delta v})
    ={\Delta v}^\top Q^\top Q{\Delta v}
    ={\Delta v}^\top Q^2{\Delta v}
    ={\Delta v}^\top Q{\Delta v}.
\end{align*}
\Eqref{eq:gs-line-overlap} further gives
\[
  {\Delta v}^\top Q{\Delta v}
  =A\,u^\top\widetilde u\widetilde u^\top u
  =Aa^2
  =\frac{A}{1+R}.
\]
Consequently,
\begin{align*}
  \gsnorm{\omega {\Delta v}-Q{\Delta v}}^2
  &=\gsnorm{\omega {\Delta v}}^2-2\gsip{\omega {\Delta v}}{Q{\Delta v}}+\gsnorm{Q{\Delta v}}^2\\
  &=A\omega^2-2A\omega a^2+Aa^2\\
  &=A\left[(\omega-a^2)^2+a^2-a^4\right].
\end{align*}
The last term is independent of \(\omega\), and \(A>0\), so the unique minimizer
is
\[
  \omega=a^2=\frac{1}{1+R}.
\]
Since \(R\ge0\), this value lies in \((0,1]\), so imposing the bounded family
\(0\le \omega\le1\) does not change the solution.

\paragraph{Step 6: prove the exact variational identity.}
Multiply the preceding expansion before completing the square by \(1+R\):
\begin{align}
  &(1+R)\gsnorm{\omega {\Delta v}-Q{\Delta v}}^2\nonumber\\
  &\quad=(1+R)A
    \left(\omega^2-\frac{2\omega}{1+R}+\frac{1}{1+R}\right)\nonumber\\
  &\quad=A\left[(1+R)\omega^2-2\omega+1\right].
  \label{eq:gs-projection-loss-expanded}
\end{align}
The first term on the right side of
\eqref{eq:gs-projective-variational-identity} is
\[
  \gsnorm{{\Delta v}-\omega {\Delta v}}^2=(1-\omega)^2A=A(1-2\omega+\omega^2).
\]
Linearity of \(P_{\Delta v^{-}}^\perp\) and
\eqref{eq:gs-projector-identity} give
\begin{align*}
  \rho_{q,n}^2\gsnorm{P_{\Delta v^{-}}^\perp(\omega {\Delta v})}^2
  &=\rho_{q,n}^2\omega^2\gsnorm{P_{\Delta v^{-}}^\perp {\Delta v}}^2\\
  &=\rho_{q,n}^2\omega^2A\sin^2\theta_q\\
  &=AR\omega^2.
\end{align*}
Adding these two expressions yields
\begin{align*}
  &\gsnorm{{\Delta v}-\omega {\Delta v}}^2
   +\rho_{q,n}^2\gsnorm{P_{\Delta v^{-}}^\perp(\omega {\Delta v})}^2\\
  &\quad=A(1-2\omega+\omega^2)+AR\omega^2\\
  &\quad=A\left[(1+R)\omega^2-2\omega+1\right],
\end{align*}
which equals \eqref{eq:gs-projection-loss-expanded}. This proves
\eqref{eq:gs-projective-variational-identity} exactly. In particular,
the objective on the right-hand side of
\eqref{eq:gs-projective-variational-identity}
is a positive rescaling of one 
euclidean projection loss, and its two displayed terms are its exact scalar
expansion. They need not be interpreted as orthogonal components of the
original output-space residual.
Completing the square gives the equivalent certificate
\begin{equation}
  \mathcal J_{q,n}(\omega)
  =A(1+R)\left(\omega-\frac1{1+R}\right)^2
   +\frac{AR}{1+R}.
  \label{eq:gs-objective-square}
\end{equation}
The nonnegative square vanishes only at the GeoShrink coefficient, proving
optimality and uniqueness a second way.

\paragraph{Sensitivity to reference amplitude and tangent scale.}
The same projection geometry identifies precisely how two deliberate method
changes alter the coefficient. Let \(\beta\in\gsR\) multiply the reference
innovation, and rescale the chordal tangent by \(\sqrt\kappa\), where
\(\kappa\ge0\). Define
\[
  \widetilde u_{\kappa}
  =\frac{u-\sqrt\kappa\rho_{q,n}\xi}
         {\sqrt{1+\kappa R}},
  \qquad
  Q_{\kappa}=\widetilde u_{\kappa}\widetilde u_{\kappa}^\top.
\]
Repeating \eqref{eq:gs-line-overlap} gives
\[
  \gsip u{\widetilde u_{\kappa}}^2=\frac{1}{1+\kappa R}.
\]
Therefore
\begin{equation}
  \arg\min_{\omega\in\gsR}\gsnorm{\omega {\Delta v}-\beta Q_{\kappa}{\Delta v}}^2
  =\frac{\beta}{1+\kappa R}.
  \label{eq:gs-general-reference-solution}
\end{equation}
If the candidate family is constrained to \(0\le \omega\le1\), the solution is
its Euclidean projection onto that interval:
\begin{equation}
  \omega^\star_{[0,1]}
  =\gsclip\!\left(\frac{\beta}{1+\kappa R}\right).
  \label{eq:gs-general-reference-clipped}
\end{equation}
The corresponding exact identity is
\begin{align}
  &(1+\kappa R)\gsnorm{\omega {\Delta v}-\beta Q_{\kappa}{\Delta v}}^2\nonumber\\
  &\quad=A(\beta-\omega)^2
    +\kappa\rho_{q,n}^2\gsnorm{P_{\Delta v^{-}}^\perp(\omega {\Delta v})}^2.
  \label{eq:gs-general-reference-objective}
\end{align}
This is a counterfactual sensitivity calculation, rather than an additional
assumption used by GeoShrink. GeoShrink uses the latest observed innovation
without amplitude rescaling and the unscaled ambient chordal tangent, hence
\((\beta,\kappa)=(1,1)\). A
time-scaled reference sets \(\beta=\lambda_{q,n}\). A value
\(\kappa\ne1\) changes the tangent transport scale and is a different method;
it is not required to derive the implemented rule.

\paragraph{Relation to intrinsic projective geometry.}
Let
\[
  \theta_q=\arccos|\gsip uv|\in[0,\pi/2].
\]
Then
\[
  \gsnorm\xi=\sin\theta_q
  =\theta_q+O(\theta_q^3)
  \qquad\text{as }\theta_q\to0.
\]
Thus the chordal tangent and the intrinsic projective logarithm agree to first
order in the fine-anchor regime covered by Theorem~\ref{thm:gs-continuum}.
They differ at finite angles. If the intrinsic logarithm is followed by the
same normalization retraction and round-trip projection, define, for
\(\theta_q>0\), the tangent pointing away from the aligned previous line by
\[
  \zeta=-\frac{\theta_q}{\sin\theta_q}\xi.
\]
Because \(\xi\perp u\) and \(\gsnorm\xi=\sin\theta_q\),
\[
  \zeta\perp u,
  \qquad
  \gsnorm\zeta=\theta_q.
\]
The normalization retraction therefore gives
\[
  u_{\rm log}
  =\frac{u+\rho_{q,n}\zeta}
         {\sqrt{1+\rho_{q,n}^2\theta_q^2}},
  \qquad
  \gsip u{u_{\rm log}}^2
  =\frac{1}{1+\rho_{q,n}^2\theta_q^2}.
\]
Step 4 showed that the round-trip projection coefficient equals this squared
line overlap. Hence
\[
  \omega_{\rm log}=\frac{1}{1+\rho_{q,n}^2\theta_q^2}.
\]
At \(\theta_q=0\), the same formula follows by continuity and equals one.

If instead one transports by the exact projective exponential along the
same local geodesic branch, its unit-sphere representative is
\[
  u_{\rm exp}
  =\cos(\rho_{q,n}\theta_q)u
   +\sin(\rho_{q,n}\theta_q)\frac{\zeta}{\theta_q}.
\]
Orthogonality of \(u\) and \(\zeta\) gives
\[
  \gsip u{u_{\rm exp}}
  =\cos(\rho_{q,n}\theta_q).
\]
Applying the same round-trip projection identity then yields
\[
  \omega_{\rm exp}=\cos^2(\rho_{q,n}\theta_q),
\]
for horizons that remain on that branch. These are different finite rules.
We call the present construction chordal or extrinsic throughout because its
normalization retraction and round-trip projection give exactly
\((1+\rho_{q,n}^2\sin^2\theta_q)^{-1}\) in the zero-regularization model.

\paragraph{Degenerate cases and scope.}
If \({\Delta v}=0\), every vector \(\omega {\Delta v}\) is zero, so the correction is zero and the
prediction is \(v_n^{(1)}\); no current line or scalar coefficient needs to be
defined. If \({\Delta v^{-}}=0\), the preceding line and its chordal tangent are
unidentified, so the ideal theorem does not assign a geometric coefficient;
a numerical fallback in this event is an operational convention. If
\(\sin^2\theta_q=0\), then \(\xi=0\), \(\widetilde L=L\), and
\(\omega=1\). If \(\rho_{q,n}=0\), the retracted line also equals \(L\) and
\(\omega=1\).

The theorem establishes exact optimality for the observable chordal
transport-and-projection closure. It does not identify the true future line.
Proposition~\ref{prop:gs-nonidentifiability} shows why no such identification
follows from finite anchor data and smoothness alone; the causal bound in
\eqref{eq:gs-causal-bound} states separately what standard
\(C^{1,1}\) regularity does guarantee about the actual trace error.

\subsection{Blockwise Direct-Sum Extension}
\label{app:gs-blockwise}

Let the output space be an orthogonal direct sum
\(\mathcal H=\bigoplus_{b=1}^{G}\mathcal H_b\), and write
\({\Delta v}=({\Delta v}^{(1)},\ldots,{\Delta v}^{(G)})\). A modality, a batch element,
or any orthogonal coordinate/token group may define one block. No statistical
independence assumption is required. For block-specific
horizons \(\rho_b\), previous innovations \({\Delta v^{(b),-}}\), and coefficients \(\omega_b\),
the ideal objective below assumes \({\Delta v^{(b),-}}\ne0\) for every block. If
\({\Delta v^{(b),-}}=0\), the ideal projector term is undefined and the residual convention
of Appendix~\ref{app:gs-numerics} should be used. If
\(A_b=\gsnorm{{\Delta v}^{(b)}}^2=0\), the corresponding block correction is
zero for every coefficient, so that coefficient is immaterial. Under the
nondegenerate ideal convention, consider
\begin{equation}
  \mathcal J_{\rm block}(\omega_1,\ldots,\omega_G)
  =\sum_{b=1}^{G}\left[
    \gsnorm{{\Delta v}^{(b)}-\omega_b{\Delta v}^{(b)}}^2
    +\rho_b^2\gsnorm{P_{{\Delta v^{(b),-}}}^\perp(\omega_b{\Delta v}^{(b)})}^2
  \right].
  \label{eq:gs-block-objective}
\end{equation}

\begin{gsproposition}[Blockwise GeoShrink]
\label{prop:gs-blockwise}
For every block with
\(A_b=\gsnorm{{\Delta v}^{(b)}}^2>0\) and \({\Delta v^{(b),-}}\ne0\), the unique minimizer of
\eqref{eq:gs-block-objective} is
\begin{equation}
  \omega_b^\star
  =\frac{A_b}{A_b+\rho_b^2N_b}
  =\frac{1}{1+\rho_b^2\sin^2\theta_b},
  \quad
  N_b=\gsnorm{P_{{\Delta v^{(b),-}}}^\perp{\Delta v}^{(b)}}^2,
  \quad
  \sin^2\theta_b=N_b/A_b.
  \label{eq:gs-block-solution}
\end{equation}
If a single shared coefficient \(\omega\) is imposed across all blocks, its
unique minimizer, provided \(\sum_b A_b>0\), is
\begin{equation}
  \omega_{\rm shared}^\star
  =\frac{\sum_b A_b}
         {\sum_b A_b+\sum_b\rho_b^2N_b}.
  \label{eq:gs-shared-solution}
\end{equation}
\end{gsproposition}

\begin{proof}
Orthogonality of the direct sum makes the squared norm of a concatenated
vector equal to the sum of its block squared norms. Therefore
\eqref{eq:gs-block-objective} contains no cross-block terms. For
each block, the expansion in \eqref{eq:gs-projective-variational-identity}
gives
\[
  \mathcal J_b(\omega_b)
  =A_b(1-\omega_b)^2+\rho_b^2N_b\omega_b^2.
\]
Differentiation gives
\[
  \mathcal J_b'(\omega_b)
  =2(A_b+\rho_b^2N_b)\omega_b-2A_b,
\]
and the second derivative
\(2(A_b+\rho_b^2N_b)\) is positive. Solving the first-order equation proves
\eqref{eq:gs-block-solution}. If all blocks share \(\omega\), summing
the displayed quadratic over \(b\) gives
\[
  \mathcal J_{\rm shared}(\omega)
  =\left(\sum_bA_b\right)(1-\omega)^2
   +\omega^2\sum_b\rho_b^2N_b.
\]
Its derivative is
\[
  2\left(\sum_bA_b+\sum_b\rho_b^2N_b\right)\omega
  -2\sum_bA_b.
\]
Setting it to zero proves \eqref{eq:gs-shared-solution}; the
positive second derivative proves uniqueness.
\end{proof}

\subsection{Innovation Retention and Time-Scaled Extrapolation}
\label{app:gs-time-scaling}

The same dimensionless time displacement can play two different roles.
GeoShrink uses the update-midpoint displacement \(\rho_{q,n}\) only as the
strength of the transported-turning penalty. An affine extrapolator for the
model output at the stage time \(\sigma_n\) would instead multiply the latest
innovation by
\begin{equation}
 \lambda_{q,n}
 =\frac{\sigma_n-\sigma_n^{(1)}}
        {\sigma_n^{(1)}-\sigma_n^{(2)}}
 \label{eq:gs-time-ratio}
\end{equation}
to account for the target displacement. The two quantities share the same
anchor-span denominator, but \(\rho_{q,n}\) uses the solver-interval midpoint
whereas \(\lambda_{q,n}\) uses the field-evaluation time. More fundamentally,
GeoShrink uses the displacement to control confidence and deliberately does
not use it as an extrapolation amplitude.

For example, let \(g(t)=b+t v\), \(v\ne0\), with ordered anchor times
\(t_{q-2}<t_{q-1}<t_q\). Both innovations are parallel to \(v\), so the
ideal \(\sin^2\theta_q=0\). GeoShrink then predicts
\[
 \widehat g(t)=g(t_q)+g(t_q)-g(t_{q-1})
             =g(t_q)+(t_q-t_{q-1})v.
\]
This equals \(g(t)\) only when
\(t-t_q=t_q-t_{q-1}\). The ideal predictor therefore does not reproduce
an arbitrary affine trace at every target time. This is a statement about
affine exactness, not a proof that the method cannot converge when all
anchor gaps and forecast distances tend to zero.

A time-scaled comparison rule would be
\begin{equation}
 \widehat v_n^{\rm time}
 =v_n^{(1)}+\frac{\lambda_{q,n}{\Delta v_q}}{1+R_{q,n}},
 \label{eq:gs-time-comparison}
\end{equation}
which is affine-exact in the ideal zero-turning case. It belongs to a
different correction family, based on
\(\lambda_{q,n}{\Delta v_q}\). For \(\lambda_{q,n}>1\), it can exceed one
observed innovation even though its retention coefficient is at most one.
It is not the algorithm specified in this manuscript.

Affine exactness by itself does not establish a smaller error on nonlinear
generative trajectories. A comparison between the two rules should keep
the exact-query schedule fixed. For either family, the loss decomposition
in Proposition~\ref{prop:gs-loss} applies after choosing the corresponding
nominal correction vector. Purely tangential acceleration remains
invisible to the projective statistic: any scalar trace \(g(t)=b(t)v\)
has zero ideal projective angle whenever its two innovations are nonzero.

\subsection{Causal Secant Remainder and Retention Bias}
\label{app:gs-secant}

Under the \(C^{1,1}\) assumptions stated below, let
\({\Delta v}=g(a)-g(a-h)\), \(\lambda=\delta/h\), and let \(\omega\in\gsR\). Then
\begin{equation}
  {
  \gsnorm{g(a+\delta)-g(a)-\omega{\Delta v}}
  \le \frac{M}{2}\delta(\delta+h)
       +|\lambda-\omega|\gsnorm{{\Delta v}}.}
  \label{eq:gs-causal-bound}
\end{equation}

We now prove \eqref{eq:gs-causal-bound} directly from the standard
\(C^{1,1}\) regularity used in ODE error analysis. Let \(g\) take values in
a finite-dimensional Euclidean space, or more generally in a Hilbert space,
and suppose
\begin{equation}
  \gsnorm{g'(t)-g'(s)}\le M|t-s|
  \qquad
  \text{for all }s,t\in[a-h,a+\delta].
  \label{eq:gs-lipschitz-derivative}
\end{equation}
Assume \(h>0\) and \(\delta\ge0\), and define
\[
  {\Delta v}=g(a)-g(a-h),
  \qquad
  \lambda=\frac{\delta}{h}.
\]
The fundamental theorem of calculus gives
\begin{align}
  {\Delta v}
  &=\int_{a-h}^{a}g'(u)\,du
    =\int_0^h g'(a-u)\,du,
  \label{eq:gs-backward-secant-integral}\\
  g(a+\delta)-g(a)
  &=\int_a^{a+\delta}g'(u)\,du
    =\int_0^\delta g'(a+s)\,ds.
  \label{eq:gs-forward-integral}
\end{align}
Multiply \eqref{eq:gs-backward-secant-integral} by
\(\lambda=\delta/h\), subtract it from
\eqref{eq:gs-forward-integral}, and add and subtract
\(\delta g'(a)\). This yields the exact identity
\begin{align}
  &g(a+\delta)-g(a)-\lambda{\Delta v}\nonumber\\
  &\quad=
  \int_0^\delta\bigl[g'(a+s)-g'(a)\bigr]\,ds
  +\frac{\delta}{h}\int_0^h
       \bigl[g'(a)-g'(a-u)\bigr]\,du.
  \label{eq:gs-secant-remainder-identity}
\end{align}
Apply the triangle inequality for vector-valued integrals and then
\eqref{eq:gs-lipschitz-derivative} to each term:
\begin{align*}
  &\gsnorm{g(a+\delta)-g(a)-\lambda{\Delta v}}\\
  &\quad\le
  \int_0^\delta\gsnorm{g'(a+s)-g'(a)}\,ds
  +\frac{\delta}{h}\int_0^h
       \gsnorm{g'(a)-g'(a-u)}\,du\\
  &\quad\le
  \int_0^\delta Ms\,ds
  +\frac{\delta}{h}\int_0^h Mu\,du\\
  &\quad=
  \frac{M\delta^2}{2}+\frac{M\delta h}{2}\\
  &\quad=
  \frac{M}{2}\delta(\delta+h).
\end{align*}
This is the causal secant remainder: it uses only past observations and is
exact for every affine trace.

For an arbitrary correction coefficient \(\omega\), add and subtract
\(\lambda{\Delta v}\):
\begin{align*}
  g(a+\delta)-g(a)-\omega{\Delta v}
  &=\bigl[g(a+\delta)-g(a)-\lambda{\Delta v}\bigr]
    +(\lambda-\omega){\Delta v}.
\end{align*}
Taking norms, using the triangle inequality, and substituting the bound just
proved gives
\[
  \gsnorm{g(a+\delta)-g(a)-\omega{\Delta v}}
  \le\frac{M}{2}\delta(\delta+h)
     +|\lambda-\omega|\gsnorm{\Delta v},
\]
which is \eqref{eq:gs-causal-bound}. No probabilistic independence,
Gaussian noise, monotone curvature, or assumed relation between tangential
and normal acceleration is used.

For the ideal GeoShrink coefficient \(\omega=1/(1+R)\), another triangle
inequality on the real line gives
\begin{align*}
  |\lambda-\omega|
  &\le|\lambda-1|+|1-\omega|\\
  &=|\lambda-1|+\frac{R}{1+R}.
\end{align*}
Therefore
\begin{equation}
  {
  \gsnorm{g(a+\delta)-g(a)-\omega{\Delta v}}
  \le\frac{M}{2}h^2\lambda(1+\lambda)
  +\left(|\lambda-1|+\frac{R}{1+R}\right)\gsnorm{\Delta v}.}
  \label{eq:gs-geoshrink-smoothness-bound}
\end{equation}
The three displayed contributions have separate meanings: local secant
truncation, the one-span reference mismatch, and shrinkage bias. When
\(g\) is affine, \(M=0\); when it is also collinear, \(R=0\). The remaining
error is then exactly \(|\lambda-1|\gsnorm{\Delta v}\), recovering the affine
counterexample in Appendix~\ref{app:gs-time-scaling}.

\subsection{Relative-Horizon Certificates}
\label{app:gs-horizon}

Let \(s_{j-1}<s_j<s_{j+1}\) be consecutive anchors in an increasing affine
progress coordinate and let
\[
  h_j=s_j-s_{j-1}>0,
  \qquad
  h_{j+1}=s_{j+1}-s_j>0.
\]
For any update interval wholly contained in the next anchor cell, its
midpoint \(\mu\) satisfies
\[
  s_j\le\mu\le s_{j+1}.
\]
Subtracting \(s_j\) from all three terms gives
\[
  0\le\mu-s_j\le s_{j+1}-s_j=h_{j+1}.
\]
Division by the positive previous span \(h_j\) yields
\[
  0\le\frac{\mu-s_j}{h_j}\le\frac{h_{j+1}}{h_j},
\]
which is \eqref{eq:gs-horizon-bound}.

Likewise, for any field-evaluation coordinate \(\xi\) in the same cell,
define
\[
  \lambda_j^{(s)}=\frac{\xi-s_j}{h_j}.
\]
The inequalities \(s_j\le\xi\le s_{j+1}\) give
\[
  0\le\lambda_j^{(s)}\le\frac{h_{j+1}}{h_j}.
\]
If the adjacent ratio is at most \(\Gamma\), then
\(\lambda_j^{(s)}\le\Gamma\). Since the derivative of
\(x(1+x)\) is \(1+2x>0\) for \(x\ge0\), it also follows that
\[
  \lambda_j^{(s)}(1+\lambda_j^{(s)})
  \le\Gamma(1+\Gamma).
\]

On an integer progress grid, if the cell length is the integer
\(h_{j+1}\) and both endpoints are exact anchors, the skipped stage
midpoints are
\[
  s_j+\frac32,
  s_j+\frac52,
  \ldots,
  s_{j+1}-\frac12
\]
whenever \(h_{j+1}\ge2\). The exact largest skipped midpoint horizon is
therefore
\begin{equation}
  \max_{\text{skips in cell}}|\rho_j^{(s)}|
  =\frac{h_{j+1}-1/2}{h_j}
  <\frac{h_{j+1}}{h_j}.
  \label{eq:gs-discrete-horizon}
\end{equation}
If there is no skipped stage in the cell, the maximum over skips is zero.

Finally, \(0\le\sin^2\theta_q\le1\) implies in this progress coordinate
\[
  R_{q,n}^{(s)}=(\rho_{q,n}^{(s)})^2\sin^2\theta_q
  \le\left(\frac{h_{j+1}}{h_j}\right)^2.
\]
Since \(\omega^{(s)}=1/(1+R^{(s)})\) is decreasing in \(R^{(s)}\), if
\(h_{j+1}/h_j\le\Gamma\), then
\[
  \omega_{q,n}^{(s)}\ge\frac{1}{1+\Gamma^2}.
\]
This proves \eqref{eq:gs-retention-certificate} and establishes
the direct connection between adjacent gap expansion and the worst
model-free retention certificate.

\subsection{The Geometric Mesh at Fixed First Span and Coverage}
\label{app:gs-geometric}

\begin{gstheorem}[Minimax normalized horizon at fixed coverage]
\label{thm:gs-geometric}
Fix \(N\ge2\), \(d>0\), and \(H\ge Nd\). Among positive spans
\(\ell_0,\ldots,\ell_{N-1}\) satisfying
\(\ell_0=d\) and \(\sum_{j=0}^{N-1}\ell_j=H\), the unique minimizer of
\begin{equation}
  \Gamma(\ell)=
  \max_{0\le j<N-1}\frac{\ell_{j+1}}{\ell_j}
  \label{eq:gs-ratio-objective}
\end{equation}
is
\begin{equation}
  {
  \ell_j=dr^j,\qquad
  d\sum_{j=0}^{N-1}r^j=H,\qquad r\ge1.}
  \label{eq:gs-geometric-closure}
\end{equation}
The closure equation has exactly one root in \([1,\infty)\).
\end{gstheorem}

\begin{gscorollary}[Canonical terminal-closed mesh after minimal bootstrap]
\label{cor:gs-whole-mesh}
Fix \(B\ge4\), a total progress length \(S\), and a bootstrap span \(d>0\),
with \(S\ge(B-1)d\). Among terminal-closed meshes with positive gaps
\(h_1,\ldots,h_{B-1}\),
\begin{equation}
  h_1=h_2=d,\qquad
  \sum_{j=1}^{B-1}h_j=S,
  \label{eq:gs-whole-mesh-constraints}
\end{equation}
the unique minimizer of
\(\max_{2\le j\le B-2}h_{j+1}/h_j\) is
\begin{equation}
  {
  h_1=d,\qquad h_j=dr^{j-2}\ (2\le j\le B-1),\qquad
  d\left(1+\sum_{k=0}^{B-3}r^k\right)=S,\quad r\ge1.}
  \label{eq:gs-whole-mesh}
\end{equation}
The corresponding continuous anchors are
\begin{equation}
  \widetilde s_0=0,\qquad \widetilde s_1=d,\qquad
  \widetilde s_q=d\left(1+\sum_{k=0}^{q-2}r^k\right),
  \quad 2\le q\le B-1.
  \label{eq:gs-whole-anchors}
\end{equation}
\end{gscorollary}

\begin{proof}[Proof of Theorem~\ref{thm:gs-geometric}]
For \(N\ge2\), the polynomial
\[
 S_N(r)=\sum_{j=0}^{N-1}r^j
\]
is continuous on \([1,\infty)\). Its derivative is
\[
 S_N'(r)=\sum_{j=1}^{N-1}jr^{j-1}.
\]
Every summand is nonnegative on \([1,\infty)\), and the \(j=1\) summand
equals one. Hence \(S_N'(r)>0\), so \(S_N\) is strictly increasing. At the
left endpoint,
\[
 S_N(1)=\sum_{j=0}^{N-1}1=N.
\]
The last term \(r^{N-1}\) diverges as \(r\to\infty\), so
\(S_N(r)\to\infty\). Because \(H/d\ge N=S_N(1)\), the intermediate value
theorem gives at least one \(r\in[1,\infty)\) satisfying
\(S_N(r)=H/d\). Strict monotonicity gives at most one such \(r\). Thus the
closure root exists and is unique.

For any feasible spans, let
\(\gamma=\max_{0\le j<N-1}\ell_{j+1}/\ell_j\).
By the definition of a maximum,
\[
 \frac{\ell_{j+1}}{\ell_j}\le\gamma
 \quad\Longrightarrow\quad
 \ell_{j+1}\le\gamma\ell_j
 \qquad (0\le j<N-1).
\]
We now prove \(\ell_j\le d\gamma^j\) by induction. For \(j=0\), feasibility
gives \(\ell_0=d=d\gamma^0\). If the bound holds at index \(j\), then
\[
 \ell_{j+1}\le\gamma\ell_j
 \le\gamma(d\gamma^j)=d\gamma^{j+1}.
\]
Thus it holds for every \(j=0,\ldots,N-1\). Summing the componentwise
bounds yields
\begin{equation}
 H=\sum_{j=0}^{N-1}\ell_j
 \le\sum_{j=0}^{N-1}d\gamma^j
 =dS_N(\gamma).
 \label{eq:gs-geometric-majorization}
\end{equation}
If \(\gamma<1\), the sum is strictly smaller than \(Nd\), contradicting
\(H\ge Nd\). Explicitly, \(\gamma^0=1\) and
\(\gamma^j<1\) for every \(j\ge1\), so
\[
 dS_N(\gamma)
 =d\left(1+\sum_{j=1}^{N-1}\gamma^j\right)
 <d[1+(N-1)]=Nd.
\]
\Eqref{eq:gs-geometric-majorization} would then imply
\(H<Nd\), a contradiction. Hence \(\gamma\ge1\). On this interval,
strict monotonicity of \(S_N\), together with
\[
 dS_N(\gamma)\ge H=dS_N(r),
\]
implies \(S_N(\gamma)\ge S_N(r)\) and therefore \(\gamma\ge r\).

Now set \(\ell_j^\star=dr^j\). Its first span is
\(\ell_0^\star=d\), and the closure equation gives
\[
 \sum_{j=0}^{N-1}\ell_j^\star
 =dS_N(r)=H,
\]
so it is feasible. Every adjacent ratio is
\[
 \frac{\ell_{j+1}^\star}{\ell_j^\star}
 =\frac{dr^{j+1}}{dr^j}=r,
\]
so its objective value is \(r\). The lower bound \(\gamma\ge r\) proves
global optimality.

For uniqueness, suppose another feasible sequence has objective value
\(r\). Repeating the induction above with \(\gamma=r\) gives
\[
 \ell_j\le dr^j=\ell_j^\star
 \qquad\text{for every }j.
\]
If at least one of these inequalities were strict, summing all coordinates
would give
\[
 \sum_{j=0}^{N-1}\ell_j
 <\sum_{j=0}^{N-1}dr^j
 =H,
\]
contradicting feasibility. Therefore equality holds at every coordinate,
and the geometric sequence is the unique minimizer.
\end{proof}

\begin{proof}[Proof of Corollary~\ref{cor:gs-whole-mesh}]
The first gap is fixed at \(h_1=d\), so the remaining gaps
\[
  \ell_0=h_2,\quad \ell_1=h_3,\quad\ldots,\quad
  \ell_{B-3}=h_{B-1}
\]
contain \(N=B-2\) positive spans. Their first span is
\(\ell_0=h_2=d\), and their total coverage is
\begin{align*}
  \sum_{k=0}^{B-3}\ell_k
  &=\sum_{j=2}^{B-1}h_j\\
  &=S-h_1\\
  &=S-d.
\end{align*}
The assumption \(S\ge(B-1)d\) implies
\[
  S-d\ge(B-2)d=Nd,
\]
so Theorem~\ref{thm:gs-geometric} applies with \(H=S-d\). It gives the
unique minimizing branch
\[
  \ell_k=dr^k,
  \qquad 0\le k\le B-3,
\]
where \(r\ge1\) is the unique solution of
\[
  d\sum_{k=0}^{B-3}r^k=S-d.
\]
Restoring \(h_1=d\) and using \(h_j=\ell_{j-2}\) for \(j\ge2\) yields
\[
  h_1=d,
  \qquad
  h_j=dr^{j-2}\quad(2\le j\le B-1).
\]
Moving the fixed first gap to the left side of the closure equation gives
\[
  d\left(1+\sum_{k=0}^{B-3}r^k\right)=S,
\]
which is \eqref{eq:gs-whole-mesh}. Starting from
\(\widetilde s_0=0\), the
cumulative locations are
\begin{align*}
  \widetilde s_1&=h_1=d,\\
  \widetilde s_q
  &=h_1+\sum_{j=2}^{q}h_j\\
  &=d+\sum_{j=2}^{q}dr^{j-2}\\
  &=d\left(1+\sum_{k=0}^{q-2}r^k\right),
  \qquad 2\le q\le B-1,
\end{align*}
which proves \eqref{eq:gs-whole-anchors}; at \(q=B-1\), the
closure equation gives \(\widetilde s_{B-1}=S\). The objective in the corollary
is exactly the adjacent-ratio objective of the \(\ell\)-branch, because
\[
  \max_{2\le j\le B-2}\frac{h_{j+1}}{h_j}
  =\max_{0\le k<B-3}\frac{\ell_{k+1}}{\ell_k}.
\]
Existence, attainment, and uniqueness therefore follow directly from the
corresponding parts of Theorem~\ref{thm:gs-geometric}.
\end{proof}

\subsection{Integer Projection of the Canonical Mesh}
\label{app:gs-finite}

For the canonical integer mesh, compute the continuous anchors from
\eqref{eq:gs-whole-anchors} with \(S=T-1\) and \(d=1\), fix
\begin{equation}
  s_0=0,\qquad s_1=1,\qquad s_2=2,\qquad s_{B-1}=T-1,
  \label{eq:gs-canonical-boundaries}
\end{equation}
and sequentially for \(3\le q\le B-2\) set
\begin{equation}
  {
  s_q=\min\!\left\{T-B+q,
       \max\!\left(s_{q-1}+1,Q_{\rm even}(\widetilde s_q)\right)
       \right\}.}
  \label{eq:gs-canonical-rounding}
\end{equation}
For any resulting integer mesh, define the realized whole-mesh certificate
\begin{equation}
  \widehat\Gamma(\mathcal A)
  =\max_{2\le j<B-1}\frac{s_{j+1}-s_j}{s_j-s_{j-1}}.
  \label{eq:gs-realized-certificate}
\end{equation}

We prove that \eqref{eq:gs-canonical-boundaries}--
\eqref{eq:gs-canonical-rounding} always produce a feasible integer mesh.
For every index \(q\), define the largest value that leaves enough terminal
slots by
\[
  U_q=T-1-(B-1-q)=T-B+q.
\]
Because \(T\ge B\), the last bootstrap offset satisfies
\[
  s_2=2\le T-B+2=U_2.
\]
Now fix \(q\in\{3,\ldots,B-2\}\) and suppose inductively that
\(s_{q-1}\le U_{q-1}\). Since \(U_q=U_{q-1}+1\),
\begin{align*}
  s_{q-1}+1
  &\le U_{q-1}+1\\
  &=U_q\\
  &=T-B+q.
\end{align*}
Thus the interval between the lower and upper guards in
\eqref{eq:gs-canonical-rounding} is nonempty. The rounded value
\(Q_{\rm even}(\widetilde s_q)\) is an integer, and both guards preserve
integrality, so \(s_q\) is an integer satisfying
\begin{equation}
  s_{q-1}+1\le s_q\le U_q.
  \label{eq:gs-canonical-rounding-induction}
\end{equation}
This proves by induction that every interior offset is strictly larger than
its predecessor and leaves the required number of unused integer stages.
At the final interior index, if it exists,
\[
  s_{B-2}\le U_{B-2}=T-2<T-1=s_{B-1}.
\]
When the interior range is empty, \(B=4\), and the same terminal inequality
follows directly from \(s_2=2\le T-2\). Consequently,
\[
  0=s_0<s_1=1<s_2=2<\cdots<s_{B-1}=T-1
\]
contains exactly \(B\) distinct integer offsets. The mapping
\(a_{q+1}=T-s_q\) reverses their order and therefore gives exactly \(B\)
distinct solver-stage anchors containing \(T,T-1,T-2\), and \(1\).

The full-grid boundary case is also contained in the same construction. If
\(B=T\), then \(U_q=q\). Starting from \(s_2=2\),
\eqref{eq:gs-canonical-rounding-induction} forces
\(s_q=q\) at every interior index, and the terminal condition gives
\(s_{T-1}=T-1\). Hence every solver stage is an anchor.

For completeness, the realized certificate applies to any strictly
increasing integer mesh satisfying
\(s_0=0,s_1=1,s_2=2,s_{B-1}=T-1\), independently of how it was obtained.
Define
\(\widehat h_j=s_j-s_{j-1}>0\). By
\eqref{eq:gs-realized-certificate},
\[
  \frac{\widehat h_{j+1}}{\widehat h_j}
  \le\widehat\Gamma(\mathcal A)
  \qquad(2\le j<B-1).
\]
Applying the cell calculation in Appendix~\ref{app:gs-horizon} at every
eligible anchor gives
\[
  \lambda_j^{(s)}\le\widehat\Gamma(\mathcal A),
  \qquad
  \lambda_j^{(s)}(1+\lambda_j^{(s)})
  \le\widehat\Gamma(\mathcal A)
       [1+\widehat\Gamma(\mathcal A)],
\]
and, because \(0\le\sin^2\theta_j\le1\),
\[
  R_j^{(s)}\le\widehat\Gamma(\mathcal A)^2,
  \qquad
  \omega_j^{(s)}\ge
  \frac{1}{1+\widehat\Gamma(\mathcal A)^2}.
\]
This proof uses the realized gaps themselves and therefore remains valid
when rounding changes the continuous geometric ratios.

\subsection{Comparison under a Bounded Time Warp}
\label{app:gs-time-warp}

Suppose a monotone time coordinate can be written as \(t=\psi(s)\),
with \(0<m\le\psi'(s)\le M_\psi<\infty\) over the relevant intervals.
For a decreasing solver time, use \(t=-\sigma\); this leaves the ratio
in \eqref{eq:gs-rho} unchanged.
Let \(r_0,r_1\) be the progress endpoints of a skipped raw interval
after \(s_j\), and let \(s_{j-1},s_j\) be the latest anchors.
The two midpoint-based ratios are
\begin{align*}
 \rho^{(s)}
 &=\frac{r_0+r_1-2s_j}{2(s_j-s_{j-1})},\\
 \rho^{(t)}
 &=\frac{\psi(r_0)+\psi(r_1)-2\psi(s_j)}
         {2[\psi(s_j)-\psi(s_{j-1})]}.
\end{align*}
All increments in the decomposition of the numerator as
\[
 [\psi(r_0)-\psi(s_j)]
 +[\psi(r_1)-\psi(s_j)]
\]
are nonnegative. For any \(a\le b\), the fundamental theorem of calculus
and the derivative bounds give
\[
 m(b-a)
 \le\psi(b)-\psi(a)
 =\int_a^b\psi'(u)\,du
 \le M_\psi(b-a).
\]
Apply this inequality separately to the two numerator increments. With
\[
 N_s=(r_0-s_j)+(r_1-s_j),
 \qquad
 N_t=[\psi(r_0)-\psi(s_j)]+[\psi(r_1)-\psi(s_j)],
\]
we obtain
\[
 mN_s\le N_t\le M_\psi N_s.
\]
For the denominator increments
\[
 D_s=s_j-s_{j-1},
 \qquad
 D_t=\psi(s_j)-\psi(s_{j-1}),
\]
the same argument gives
\[
 mD_s\le D_t\le M_\psi D_s.
\]
Since all denominators are positive,
\begin{align*}
 \rho^{(t)}
 &=\frac{N_t}{2D_t}\\
 &\ge\frac{mN_s}{2M_\psi D_s}
  =\frac{m}{M_\psi}\rho^{(s)},\\
 \rho^{(t)}
 &\le\frac{M_\psi N_s}{2mD_s}
  =\frac{M_\psi}{m}\rho^{(s)}.
\end{align*}
Therefore
\begin{equation}
 {
 \frac{m}{M_\psi}\rho^{(s)}
 \le\rho^{(t)}
 \le\frac{M_\psi}{m}\rho^{(s)}.}
 \label{eq:gs-time-warp}
\end{equation}
In particular, a logical-progress gap-ratio certificate \(\Gamma\)
transfers to at most \((M_\psi/m)\Gamma\) in physical solver time under
this additional bounded-distortion assumption. Combining this with
\eqref{eq:gs-retention-certificate} gives the physical-time floor
\[
  \omega\ge\frac{1}{1+(M_\psi\Gamma/m)^2}.
\]
Without a bound on the time warp, no such uniform transfer follows
from logical geometric spacing alone.

\subsection{Regularization and Degenerate Innovations}
\label{app:gs-numerics}

Let
\[
 {\Delta v}={\Delta v_q},\qquad {\Delta v^{-}}={\Delta v_q^{-}},\qquad
 S=\gsnorm {\Delta v}^2,\qquad P=\gsnorm {\Delta v^{-}}^2,\qquad D=\gsip {\Delta v}{\Delta v^{-}}.
\]
When \(P>0\), the ideal least-squares residual is
\[
 r_0={\Delta v}-\frac{D}{P}{\Delta v^{-}},
 \qquad
 N_0=\gsnorm{r_0}^2=S-\frac{D^2}{P}.
\]
The regularized residual uses
\begin{equation}
 r_{\epsilon_p}={\Delta v}-\frac{D}{P+\epsilon_p}{\Delta v^{-}}.
 \label{eq:gs-regularized-angle}
\end{equation}
Expanding its squared norm one term at a time gives
\begin{align*}
 N_{\epsilon_p}
 &=\gsnorm{r_{\epsilon_p}}^2\\
 &=\gsnorm {\Delta v}^2
   -\frac{2D}{P+\epsilon_p}\gsip {\Delta v}{\Delta v^{-}}
   +\frac{D^2}{(P+\epsilon_p)^2}\gsnorm {\Delta v^{-}}^2\\
 &=S-\frac{2D^2}{P+\epsilon_p}
   +\frac{D^2P}{(P+\epsilon_p)^2}\\
 &=S-\frac{D^2(P+2\epsilon_p)}{(P+\epsilon_p)^2}.
\end{align*}
Subtracting the ideal residual energy yields the exact identity
\begin{equation}
 {
 N_{\epsilon_p}-N_0
 =\frac{D^2\epsilon_p^2}
        {P(P+\epsilon_p)^2}\ge0.}
 \label{eq:gs-epsilon-bias}
\end{equation}
Thus regularizing the projection denominator can only increase the residual
energy relative to exact projection.

The finite weight is itself an exact minimizer. Define
\[
 \mathcal J_\epsilon(\omega)
 =S(1-\omega)^2+(\rho^2N_{\epsilon_p}+\epsilon_s)\omega^2.
\]
Expanding and differentiating gives
\[
 \mathcal J_\epsilon'(\omega)
 =2[S+\rho^2N_{\epsilon_p}+\epsilon_s]\omega-2S,
\]
and
\[
 \mathcal J_\epsilon''(\omega)
 =2[S+\rho^2N_{\epsilon_p}+\epsilon_s].
\]
This is strictly positive unless all three nonnegative terms vanish. Hence
the unique minimizer whenever the denominator is positive is
\begin{equation}
 \omega_\epsilon
 =\frac{S}{S+\rho^2N_{\epsilon_p}+\epsilon_s}.
 \label{eq:gs-regularized-weight}
\end{equation}
This shows precisely that \(\epsilon_s\) adds a small numerical ridge to the
transported-turning penalty.

The degenerate cases require an explicit convention only when the displayed
ratio becomes \(0/0\). If \({\Delta v}=0\), then \(S=N_{\epsilon_p}=0\). When
\(\epsilon_s>0\), the formula gives \(\omega_\epsilon=0\); when
\(\epsilon_s=0\), every \(\omega\) has the same objective and the implementation
sets \(\omega_\epsilon=0\). In both cases the correction \(\omega_\epsilon {\Delta v}\) is
zero. If
\({\Delta v^{-}}=0\), then \(P=D=0\), so \(r_{\epsilon_p}={\Delta v}\): without a previous line,
the full latest innovation is treated as unsupported. If \({\Delta v^{-}}\ne0\) and
\(\epsilon_p,\epsilon_s\downarrow0\), then
\(N_{\epsilon_p}\to N_0=S\sin^2\theta\) and
\(\omega_\epsilon\to1/(1+\rho^2\sin^2\theta)\).

A backend that regularizes the direct angle instead may use
\[
 \sin^2\widetilde\theta_{\epsilon_a}
 =1-\frac{D^2}{SP+\epsilon_a}.
\]
Cauchy-Schwarz gives \(0\le D^2\le SP\), so this quantity lies in
\([0,1]\). For \(S,P>0\), its exact relation to the ideal statistic is
\begin{equation}
 \sin^2\widetilde\theta_{\epsilon_a}
 =\sin^2\theta+(1-\sin^2\theta)\frac{\epsilon_a}{SP+\epsilon_a}.
 \label{eq:gs-direct-angle-bias}
\end{equation}
Both finite-precision forms converge to the same ideal rule as their
regularizers vanish, but they are not algebraically identical at fixed
positive regularization.

Fixed absolute regularizers break the exact scale invariance of
\eqref{eq:gs-chi} and are outside the continuum theorem. Dot
products and norm accumulations should use sufficient precision, and a
direct-angle implementation should clamp its final statistic to \([0,1]\)
to suppress roundoff outside the mathematical range. All invariance and
continuum claims in the paper concern the zero-regularization equations;
reported inference uses the finite rule of its stated backend.

\subsection{Full-Grid Error Propagation}
\label{app:gs-propagation}

For an exact anchor the supplied field is exact at the current accelerated
state. Define the additional field defect by
\[
  \varepsilon_n
  =\widetilde v_n-V_\theta(\widetilde x_n,\sigma_n,c),
\]
so that \(\varepsilon_n=0\) at every anchor. For a solver with history, let
\(X_n\) denote the augmented numerical state and let \(\pi X_n=x_n\) extract
the latent.

\begin{gsproposition}[Full-grid perturbation bound]
\label{prop:gs-propagation}
Let \(\mathcal U_n(X,v)\) be the solver update and
\[
  F_n(X)=\mathcal U_n(X,V_\theta(\pi X,\sigma_n,c)).
\]
Assume
\[
  \|F_n(X)-F_n(Y)\|\le\alpha_n\|X-Y\|,
  \qquad
  \|\mathcal U_n(X,v)-\mathcal U_n(X,w)\|
  \le\beta_n\gsnorm{v-w},
\]
with \(\alpha_n,\beta_n\ge0\). If the exact-field and GeoShrink runs have
identical initial solver states, then
\begin{equation}
  {
  \|\widetilde X_0-X_0^\star\|
  \le\sum_{n=1}^{T}\beta_n\gsnorm{\varepsilon_n}
       \prod_{k=1}^{n-1}\alpha_k.}
  \label{eq:gs-general-propagation}
\end{equation}
If \(h_n=|\sigma_n-\sigma_{n-1}|\),
\(\alpha_n\le1+Lh_n\), and
\(\beta_n\le C_{\mathcal D}h_n\), then
\begin{equation}
  \|\widetilde X_0-X_0^\star\|
  \le C_{\mathcal D}e^{L\sum_n h_n}
       \sum_{n=1}^{T}h_n\gsnorm{\varepsilon_n}.
  \label{eq:gs-special-propagation}
\end{equation}
If latent extraction is nonexpansive and the exact-field full-grid method has
latent discretization error \(E_{\rm grid}\), then
\begin{equation}
  {
  E_{\rm GeoShrink}
  \le E_{\rm grid}
     +C_{\mathcal D}e^{L\sum_n h_n}
       \sum_n h_n\gsnorm{\varepsilon_n}.}
  \label{eq:gs-error-decomposition}
\end{equation}
\end{gsproposition}

\begin{proof}[Proof of Proposition~\ref{prop:gs-propagation}]
Let \(e_n=\|\widetilde X_n-X_n^\star\|\).
At stage \(n\), insert the exact-field update evaluated at the
GeoShrink solver state:
\begin{align*}
 e_{n-1}
 &\le
 \|\mathcal U_n(\widetilde X_n,\widetilde v_n)
     -\mathcal U_n(\widetilde X_n,
            V_\theta(\pi\widetilde X_n,\sigma_n,c))\|\\
 &\quad+
 \|F_n(\widetilde X_n)-F_n(X_n^\star)\|\\
 &\le\beta_n\gsnorm{\varepsilon_n}+\alpha_n e_n.
\end{align*}
The initial states coincide, so \(e_T=0\). Repeated substitution in
descending stage order can be written explicitly. Put
\(b_n=\beta_n\gsnorm{\varepsilon_n}\). The first two substitutions are
\begin{align*}
 e_0
 &\le b_1+\alpha_1e_1\\
 &\le b_1+\alpha_1b_2+\alpha_1\alpha_2e_2.
\end{align*}
We claim that after substituting through stage \(m\),
\begin{equation}
 e_0
 \le\sum_{n=1}^{m}b_n\prod_{k=1}^{n-1}\alpha_k
   +e_m\prod_{k=1}^{m}\alpha_k.
 \label{eq:gs-propagation-induction}
\end{equation}
For \(m=1\), this is the one-step recurrence, with the empty product in the
first term equal to one. Suppose it holds for \(m<T\). The next recurrence
is \(e_m\le b_{m+1}+\alpha_{m+1}e_{m+1}\). Multiplying it by the
nonnegative product \(\prod_{k=1}^{m}\alpha_k\) and substituting into
\eqref{eq:gs-propagation-induction} gives
\begin{align*}
 e_0
 &\le\sum_{n=1}^{m}b_n\prod_{k=1}^{n-1}\alpha_k
   +b_{m+1}\prod_{k=1}^{m}\alpha_k
   +e_{m+1}\prod_{k=1}^{m+1}\alpha_k\\
 &=\sum_{n=1}^{m+1}b_n\prod_{k=1}^{n-1}\alpha_k
   +e_{m+1}\prod_{k=1}^{m+1}\alpha_k.
\end{align*}
This proves the claim by induction. Taking \(m=T\) and using \(e_T=0\)
gives
\[
 e_0\le
 \sum_{n=1}^{T}\beta_n\gsnorm{\varepsilon_n}
       \prod_{k=1}^{n-1}\alpha_k,
\]
which is \eqref{eq:gs-general-propagation}.
If \(\alpha_n\le1+Lh_n\) and \(\beta_n\le C_{\mathcal D}h_n\),
then
\begin{align*}
 \prod_{k=1}^{n-1}\alpha_k
 &\le\prod_{k=1}^{n-1}(1+Lh_k)\\
 &\le\prod_{k=1}^{n-1}e^{Lh_k}\\
 &=\exp\left(L\sum_{k=1}^{n-1}h_k\right)\\
 &\le e^{L\sum_{k=1}^{T}h_k},
\end{align*}
where \(1+x\le e^x\) was applied to each nonnegative \(x=Lh_k\).
Substitution into \eqref{eq:gs-general-propagation}, together with
\(\beta_n\le C_{\mathcal D}h_n\), gives
\begin{align*}
 e_0
 &\le\sum_{n=1}^{T}C_{\mathcal D}h_n
       \gsnorm{\varepsilon_n}e^{L\sum_kh_k}\\
 &=C_{\mathcal D}e^{L\sum_kh_k}
   \sum_{n=1}^{T}h_n\gsnorm{\varepsilon_n},
\end{align*}
which proves \eqref{eq:gs-special-propagation}.

For the final error decomposition, let \(x_0\) denote the endpoint of the
continuous reference dynamics and let
\(E_{\rm grid}=\gsnorm{\pi X_0^\star-x_0}\). The triangle inequality and
nonexpansiveness of \(\pi\) give
\begin{align*}
 \gsnorm{\pi\widetilde X_0-x_0}
 &\le\gsnorm{\pi\widetilde X_0-\pi X_0^\star}
      +\gsnorm{\pi X_0^\star-x_0}\\
 &\le\gsnorm{\widetilde X_0-X_0^\star}+E_{\rm grid}.
\end{align*}
Applying \eqref{eq:gs-special-propagation} to the first term proves
\eqref{eq:gs-error-decomposition}.
\end{proof}

The augmented-state formulation is required when the numerical update
depends on stored fields or past latent states. An update written only
as a function of the current latent can otherwise omit a source of
perturbation. The special constants \(1+Lh_n\) and
\(C_{\mathcal D}h_n\) are assumptions to verify for the chosen method
and state norm; they do not follow automatically from introducing an
augmented state.

When the special constants are deterministic and the defects have finite
second moments, weighted Cauchy-Schwarz also gives, with
\(S=\sum_n h_n\),
\begin{equation}
 {
 \gsE\|\widetilde X_0-X_0^\star\|^2
 \le C_{\mathcal D}^2 e^{2LS}\,
       S\sum_n h_n\,\gsE\gsnorm{\varepsilon_n}^2.}
 \label{eq:gs-mean-square-propagation}
\end{equation}
Indeed, weighted Cauchy-Schwarz gives, pathwise,
\begin{align*}
 \left(\sum_n h_n\gsnorm{\varepsilon_n}\right)^2
 &=\left(\sum_n \sqrt{h_n}\,
          \sqrt{h_n}\gsnorm{\varepsilon_n}\right)^2\\
 &\le\left(\sum_n h_n\right)
       \left(\sum_n h_n\gsnorm{\varepsilon_n}^2\right)\\
 &=S\sum_n h_n\gsnorm{\varepsilon_n}^2.
\end{align*}
Squaring \eqref{eq:gs-special-propagation}, applying this bound,
and then taking expectations gives
\begin{align*}
 \gsE\gsnorm{\widetilde X_0-X_0^\star}^2
 &\le C_{\mathcal D}^2e^{2LS}
   \gsE\left(\sum_n h_n\gsnorm{\varepsilon_n}\right)^2\\
 &\le C_{\mathcal D}^2e^{2LS}
   S\sum_n h_n\gsE\gsnorm{\varepsilon_n}^2,
\end{align*}
where the finite sum is exchanged with expectation by linearity. This is
\eqref{eq:gs-mean-square-propagation}.
This connects local field MSE to an upper bound on deviation from the
full-grid trajectory under the stated stability conditions.
It is not a global optimality result for the retention policy:
different policies change later states, histories, and target
innovations.

\end{document}